%% file: main.tex
\documentclass{article}

\usepackage{iclr2027_conference,times}

\usepackage{fancyhdr}
\renewcommand{\headrulewidth}{0.4pt}

\usepackage[T1]{fontenc}
\usepackage{microtype}
\usepackage{amsmath,amssymb,amsthm}
\usepackage{fix-cm}
\usepackage{booktabs}
\usepackage{graphicx}
\usepackage{wrapfig}
\usepackage{float}
\usepackage[table]{xcolor}
\usepackage{tikz}
\usetikzlibrary{arrows.meta,calc}
\newcommand{\algmark}[1]{\tikz[remember picture,overlay]\coordinate(#1);}
\newcommand{\algstage}[4]{%
  \draw[line width=.8pt]
    ($(algR |- #1)+(3pt,1.55ex)$) -- ++(2.5pt,0) coordinate (bt)
    -- ($(algR |- #2)+(5.5pt,-.3ex)$) -- ++(-2.5pt,0);
  \draw[-{Latex[length=3.5pt,width=3pt]},line width=.8pt]
    ($(bt)!.5!($(algR |- #2)+(5.5pt,-.3ex)$)$) -- ++(9pt,0)
    node[right,inner sep=.8pt,circle,fill=black,text=white,font=\footnotesize] (n) {#3};
  \node[right,inner sep=1.5pt,align=left,font=\small] at (n.east) {#4};
}
\usepackage{algorithm}
\usepackage[noend]{algpseudocode}
\usepackage{listings}
\usepackage{tcolorbox}
\tcbuselibrary{skins,breakable,listings}
\usepackage{hyperref}
\hypersetup{
  colorlinks=true,
  citecolor=blue,
  linkcolor=blue,
  urlcolor=blue,
  linktoc=all
}
\usepackage{url}

\AddToHook{cmd/bibsection/after}{\addcontentsline{toc}{section}{\refname}}

\newcommand{\method}{\textsc{SAGE}}

\newcommand{\meanstd}[2]{\ensuremath{#1{\scriptstyle\,\pm\,#2}}}
\newcommand{\bestmeanstd}[2]{\ensuremath{\mathbf{#1}{\scriptstyle\,\boldsymbol{\pm}\,\mathbf{#2}}}}
\newcommand{\secondmeanstd}[2]{\ensuremath{\underline{#1{\scriptstyle\,\pm\,#2}}}}

\definecolor{qwenblue}{HTML}{0072B2}
\definecolor{ministralorange}{HTML}{E69F00}
\definecolor{phipurple}{HTML}{CC79A7}
\definecolor{sagegreen}{HTML}{009E73}
\definecolor{cotvermillion}{HTML}{D55E00}

\newtheorem{theorem}{Theorem}[section]

\newtheorem{lemma}[theorem]{Lemma}

\title{Self-Adapting Group of Experts for Multi-Agent Reasoning}

\author{
Mohammad Atif Quamar$^1$ \enspace Nurbek Tastan$^1$ \enspace Karthik Nandakumar$^{1,2}$ \enspace Junpei Komiyama$^{1,3}$ \\[1mm]
$^1$Mohamed bin Zayed University of Artificial Intelligence \enspace $^2$Michigan State University \enspace $^3$RIKEN AIP \\[1mm]
{\small\texttt{\{mohammad.atif,nurbek.tastan,karthik.nandakumar,junpei.komiyama\}@mbzuai.ac.ae}} \\[2mm]
\textbf{Project Page:} \url{https://www.atifquamar.com/sage-page}
}

\iclrfinalcopy

\begin{document}

\fancyhead{}              %
\fancyhead[L]{Preprint}   %

\raggedbottom
\setcounter{topnumber}{1}
\setcounter{bottomnumber}{1}
\setcounter{totalnumber}{2}
\renewcommand{\topfraction}{0.85}
\renewcommand{\bottomfraction}{0.70}
\renewcommand{\textfraction}{0.10}
\renewcommand{\floatpagefraction}{0.75}
\makeatletter
\setlength{\@fptop}{0pt}
\makeatother

\maketitle
\lhead{Preprint}

\begin{abstract}

Multi-agent systems bring together language model agents with different roles to propose, review, and refine solutions. Each agent's response depends on its model's capabilities, the reasoning strategy defined by its system prompt, and the information in its input context. Existing frameworks often adapt communication by changing this context while leaving individual prompts fixed, even when a problem calls for different skills. We study whether agents' initial responses can identify a strategy better suited to the current problem and guide its transfer to other agents. To address this, we introduce \method{} (\textbf{S}elf-\textbf{A}dapting \textbf{G}roup of \textbf{E}xperts), a training-free framework that uses answer agreement, prefix consistency, and reciprocal peer review to select a strategy donor. \method{} transfers the selected donor's reasoning strategy to the other agents while preserving their original roles. This transfer uses only the agents' original system prompts, without access to the problem or generated solutions. After strategy adaptation, agents exchange responses through a dynamic, sparse directed acyclic graph that routes information from higher-scoring agents to lower-scoring agents. Experiments across multiple agent backbones and reasoning benchmarks show that \method{} achieves higher average accuracy than the evaluated baselines. Our code is available \href{https://github.com/atifquamar07/sage}{here}.
\end{abstract}

\input{figures/sage_method_figure}

\section{Introduction}
\label{sec:introduction}

Large language models (LLMs) have made substantial progress in question answering, coding, and mathematical reasoning~\citep{openai2024gpt4technicalreport,grattafiori2024llama3herdmodels,qwen2025qwen25}. Yet a single response can still contain errors or fail when a problem requires several reasoning steps or knowledge beyond the model's strengths~\citep{wang2023selfconsistency,hendrycks2021mmlu}. Smaller models can excel at particular tasks while struggling with these broader demands~\citep{magister2023teaching}. These limitations motivate multi-agent systems (MAS), where several LLM-based agents discuss solutions, check one another's work, and revise their answers~\citep{li2023camel,chen2024agentverse,du2024debate}. By combining different models or roles, such systems can draw on varied expertise and solve problems that individual agents miss~\citep{chen2024reconcile,yue2025masrouter}. The challenge is to match the team's reasoning strategies and collaboration to the needs of each problem.

Recent work seeks useful, efficient communication, since more discussion does not always improve answers~\citep{smit2024mad,li2024sparsedebate}. To shape communication, GPTSwarm optimizes links and prompts, AgentPrune removes unnecessary links~\citep{zhuge2024gptswarm,zhang2025agentprune}, and G-Designer generates a graph for each problem using a learned model~\citep{zhang2025gdesigner}. 
SelfOrg rebuilds the graph each round, favoring agents whose response embeddings align with the group average (centroid)~\citep{tastan2026selforg}.
Beyond connections, MOC preserves messages from directly and indirectly connected agents without repetition~\citep{guan2026moc}, while MAD-M$^2$ filters potentially incorrect history without adapting role instructions~\citep{tian2026madmm}.

An agent's response depends on three components:
(i) its underlying model, which determines its learned capabilities;
(ii) its role-defining system instructions, which guide how it
approaches a problem; and
(iii) the problem-specific input context, including the problem statement
and any previous or peer responses.
Changing the communication structure changes which peer responses
an agent receives, but does not by itself revise its role instructions.
Although prior work also explores role and prompt adaptation,
we investigate how agents' responses to the current problem can
guide the transfer of useful reasoning instructions across roles.
This fully utilizes the capability of the agents by allowing them to adapt their reasoning strategies. 

Based on the aforementioned idea, we introduce a reference agent, which we call a \textit{strategy donor}, whose original role instructions may benefit other agents. To see how this conceptually works, consider the following physics problem: 
\begin{quote}
An object of mass $m$ is dropped from rest from a height $h$.
Air resistance opposes its motion with magnitude $kv^2$,
where $k>0$ is a constant and $v$ is its speed.
What is its speed just before it hits the ground?
\end{quote}
Suppose the agents are assigned the roles of a physicist,
a mathematician, and a computer scientist. 
The physicist's original role prompt might instruct it to identify
the relevant forces and formulate governing equations from
physical laws (Newtonian dynamics), which clarifies the underlying mechanics of the problem as a function of $m,h,k$, and $v$.
If a physicist is selected as the \textit{donor}, these general instructions can guide the adaptation of the other agents' role prompts. Given the physicist's expertise in formulating the governing equations, the other agents can focus on their specialized approaches.
The mathematician could incorporate this guidance while retaining
its expertise on analytical derivation, and the computer scientist
could focus on numerical solutions for the simulated dynamics.
Such collaboration further benefits from multi-turn interactions, in which agents iteratively refine their responses based on feedback from their peers.

We introduce \method{} (\textbf{S}elf-\textbf{A}dapting \textbf{G}roup of \textbf{E}xperts), a training-free framework that (i) selects a strategy donor from the agents' independent initial responses using answer agreement, prefix consistency, and reciprocal peer review; (ii) uses the donor's original system prompt to revise the other agents' strategies while preserving their inherent roles; (iii) constructs a sparse directed acyclic graph, rebuilt after every collaboration round, that routes responses from higher-scoring agents to lower-scoring agents. This creates a self-adaptive multi-agent system that adjusts how each individual agent reasons, adapts, and collaborates for each problem without relying on any training or external judge model.

\section{Methodology}
\label{sec:methodology}

\method{} adapts the reasoning strategies of the agents and their collaboration structure for each problem in three stages. First, a strategy donor is selected among the agents based on their independent initial responses, which are evaluated using answer agreement, prefix consistency, and reciprocal peer review (Section~\ref{sec:prompt_donor}). Second, the donor's original system prompt is used to adapt the other agents' strategies while preserving their roles (Section~\ref{sec:prompt_adaptation}). Finally, the agents propagate their responses (based on their updated system prompts) via a sparse directed acyclic communication graph and their responses are refined over multiple collaboration rounds (Section~\ref{sec:collaboration}). While prompt/strategy adaptation is performed once per problem, the communication graph is repeatedly updated between collaboration rounds. The procedure does not require any training or an external judge model. Figure~\ref{fig:sage-method-overview} illustrates the workflow, and Algorithm~\ref{alg:sage_overview} summarizes the procedure.

\paragraph{Setup.} For a given problem $x$, let $y = \mathcal{M}(x; s, c)$ be the \textit{response} of an agent, where $\mathcal{M}$ represents the agent's language model (backbone), $s$ denotes its system prompt that specifies its role and reasoning strategy (a reusable reasoning specialization such as algebra, physics, or psychology), and $c$ is the optional current context (e.g., response prefix or previous and peer responses). An agent's system prompt $s$ can change its instructions, not its model parameters. Within an agent's response $y$, we distinguish the \textit{reasoning} from the \emph{answer} $a$ (final solution to problem $x$) extracted from it. Let $\kappa(y)$ be a function that extracts and normalizes the answer $a$ from a response $y$, returning $\emptyset$ when extraction fails. 

We consider a group of $N>1$ agents, indexed by $[N] = \{1, \ldots, N\}$.  Agent $i$ uses a language model $\mathcal{M}_i$ and an initial system prompt $s_i$, $i \in [N]$. For each problem $x$, the system prompts of each agent are sampled without replacement from a shared pool. The same agents and models are used throughout inference, and no model parameters are updated. Let $y_i^t = \mathcal{M}_i(x; s_i, c_i^t)$ be the response of agent $i$ at time step $t$, where $t=0$ denotes initialization and $t \geq 1$ denotes a completed collaboration round. The context is omitted during initial generation, i.e., $y_i^0 = \mathcal{M}_i(x; s_i)$. For $t\geq1$, $s_i$ is replaced by the adapted prompt $\tilde{s}_i$ (Section~\ref{sec:prompt_adaptation}). Let $\mathcal{Y}^t=\{y_1^t, \ldots, y_N^t\}$ denote the multiset of all agent responses at step $t$.

\subsection{Selecting a Strategy Donor}
\label{sec:scoring}
\label{sec:prompt_donor}

We use the initial responses of the agents to select a \emph{strategy donor}: the agent whose original system prompt will guide the adaptation of the prompts of the other agents. Response-based donor selection makes prompt adaptation specific to the current problem.
In the mechanical physics example, this stage can select the physicist as donor when its response receives the highest score based on answer agreement, prefix consistency, and reciprocal peer review. 

\paragraph{Answer agreement.} Answer agreement measures how often the agents reach the same final answer, following the voting principle of self-consistency~\citep{wang2023selfconsistency}. Let $\mathcal{B}$ be the nonempty multiset of responses being compared, where repeated responses are counted separately. Let $\mathbb{I}$ be the indicator function. The agreement of a response $y$ with a set of responses $\mathcal{B}$ is defined as:

\begin{equation}
    q(y,\mathcal{B}) = \frac{1}{|\mathcal{B}|}\sum_{b\in\mathcal{B}}
\mathbb{I}[\kappa(b)=\kappa(y)\neq\emptyset].
\end{equation}

\paragraph{Prefix consistency.} Prefix consistency checks whether an agent reaches the same answer when asked to complete the beginning of its own response~\citep{iwase2026reliablechainofthoughtprefixconsistency}, which indicates the agent consistently reproduces its answer from partial reasoning. For a response $y$ produced by an agent, let $\pi_{\tau}(y)$ retain its initial fraction $\tau \in (0,1)$ of the response. Let $\bar y = \mathcal{M}(x; s,\pi_\tau(y))$ denote the completion generated by the same agent under the same prompt $s$ with the partial response as the input context. The prefix consistency is defined as:

\begin{equation}
    z_{\tau}(y) = \mathbb{I}[\kappa(\bar y)=\kappa(y)\neq\emptyset].
\end{equation}

The above two signals measure support across responses and reproducibility within an agent. The combined score of an agent is defined as:

\begin{equation}
    \rho(y,\mathcal{B}) =  q(y,\mathcal{B}) + \lambda z_{\tau}(y),
    \label{eq:reliability}
\end{equation}

\noindent where $\lambda\geq0$ controls the weight of prefix consistency. The first term measures support; the second records whether prefix completion recovers the same answer, which implies the answer is stable under partial re-generation.

In our multi-agent system, each agent first solves the problem independently. We score these initial responses using answer agreement and prefix consistency, then use reciprocal peer review to select the donor. Based on the initial responses of all agents $\mathcal{Y}^0$, each agent $i$ receives a score $\rho_i^0 = \rho(y^0_i,\mathcal{Y}^0)$. Prefix outcomes are sampled once per agent--response pair within a fixed-prompt phase and reused when agreement is recomputed. In summary, we score each response using agreement and stability.

\begin{algorithm}[tp]
\caption{\method{}}
\label{alg:sage_overview}
\small\hypersetup{linkcolor=black}%
\begin{minipage}[t]{0.8\linewidth}
\begin{algorithmic}[1]
\Require Problem $x$; agents $(\mathcal{M}_i,s_i)_{i=1}^N$; budgets $m,K,T$
\Ensure Final answer
\State \label{line:answer}\algmark{a1}Initial response: $y_i^0\gets\mathcal{M}_i(x;s_i)$, $\forall i\in[N]$
\State Score: $\rho_i^0\gets\rho(y_i^0,\mathcal{Y}^0)$ \Comment{Eq.~\eqref{eq:reliability}}
\State \label{line:review}Peer review: $\mathcal{R}\gets\{y_{i\leftarrow j},y_{j\leftarrow i}\}_{(i,j)\in H_m\times L_m}$
\State Retain responses with best reviews: $\mathcal{R}^\dagger\gets\{y_i^\dagger\}$ \Comment{Eq.~\eqref{eq:owner_candidate_selection}}
\State \label{line:donor}\algmark{a5}Select donor: $\ell\gets\arg\max_i\rho(y_i^\dagger,\mathcal{R}^\dagger)$ \Comment{Eq.~\eqref{eq:donor_selection}}
\State \label{line:rewrite}\algmark{a6}Rewrite system prompts: $\tilde{s}_i\gets R_{\mathcal{M}_i}(s_i,s_\ell)$, $\forall i\neq\ell$ \Comment{Eq.~\eqref{eq:prompt_rewrite}}
\For{\algmark{a7}$t=1$ to $T$} \label{line:round}
    \State Build DAG: $P_i^t$, $\forall i\in[N]$ \Comment{Eq.~\eqref{eq:parents}}
    \For{$i$ in decreasing $\rho_i^{t-1}$}
        \State Revise: $y_i^t\gets\mathcal{M}_i(x;\tilde{s}_i,(y_i^{t-1},\mathbf{y}_i^t))$ \Comment{Eq.~\eqref{eq:response_update}}
    \EndFor
    \State Rescore: $\rho_i^t\gets\rho(y_i^t,\mathcal{Y}^t)$
    \State \label{line:stop}\algmark{a12}\textbf{break} if all answers agree
\EndFor
\State \label{line:pool}\algmark{a13}Pool: $\mathcal{P}\gets(\mathcal{Y}^0,\mathcal{R}^\dagger,\mathcal{Y}^1,\ldots,\mathcal{Y}^T)$
\State \label{line:vote}\algmark{a14}\Return $\arg\max_a W(a)$ \Comment{Eq.~\eqref{eq:output}}
\end{algorithmic}
\end{minipage}\algmark{algR}%
\begin{tikzpicture}[remember picture,overlay]
\algstage{a1}{a5}{1}{Select\\donor}
\algstage{a6}{a6}{2}{Transfer\\strategy}
\algstage{a7}{a12}{3}{Collaborate}
\algstage{a13}{a14}{4}{Pool \&\\vote}
\end{tikzpicture}
\end{algorithm}

\paragraph{Reciprocal peer review and donor selection.}
We split the agents into two groups by their initial scores $\rho_i^0$. The higher-scoring group contains the agents with the highest score; if only one agent has it, the next-highest-scoring agent is added, with lower indices breaking ties. The lower-scoring group contains all remaining agents. If all scores are equal, two agents are chosen at random to form the higher-scoring group. We then sample up to $m$ agents uniformly at random from each group, giving $H_m$ and $L_m$. For each pair $(i,j)\in H_m\times L_m$, each agent reads both original responses and uses its own backbone, under a shared critic instruction, to keep or revise its response. This produces two complete candidate responses, $y_{i\leftarrow j}$ from agent $i$ and $y_{j\leftarrow i}$ from agent $j$. Let $\mathcal{R}$ be the multiset of all review candidates, $o(y)$ be the \textit{owner} agent that produced $y$, and $S=H_m\cup L_m$ be the reviewed agents. We score each candidate response using Eq.~\eqref{eq:reliability}, measuring agreement across all of $\mathcal{R}$ and checking prefix consistency under its owner's original prompt. 

We retain each agent's highest-scoring candidate response:
\begin{equation}
y_i^\dagger
=\underset{y\in\mathcal{R}:o(y)=i}{\arg\max}
\rho(y,\mathcal{R}),
\qquad i\in S.
\label{eq:owner_candidate_selection}
\end{equation}

We then recompute agreement using only the multiset of retained candidate responses, $\mathcal{R}^\dagger=\{y_i^\dagger:i\in S\}$, reuse their prefix results, and select the highest-scoring agent as donor:
\begin{equation}
\ell
=\underset{i\in S}{\arg\max}\;
\rho(y_i^\dagger,\mathcal{R}^\dagger).
\label{eq:donor_selection}
\end{equation}
Thus, each reviewed agent contributes one candidate to donor selection. Ties are resolved deterministically. The donor supplies its original system prompt $s_\ell$. While collaboration begins from the original initial responses $\mathcal{Y}^0$, only the retained reviews $\mathcal{R}^\dagger$ contribute to final answer selection.

\subsection{Adapting Strategies While Preserving Roles}
\label{sec:prompt_adaptation}

In the mechanical physics example, this stage lets the mathematician and computer scientist adopt the physicist's guidance to identify forces and formulate governing equations while retaining their analytical and numerical approaches, respectively. We consider this stage to be one of our major contributions that is not often included in other multi-agent collaboration frameworks.

For each agent $i\neq\ell$, we use its backbone $\mathcal{M}_i$ to rewrite its original prompt $s_i$ using guidance from the donor's original prompt $s_\ell$:
\begin{equation}
\tilde{s}_i
= R_{\mathcal{M}_i}\!\left(s_i,s_\ell\right),
\qquad i\neq\ell,
\label{eq:prompt_rewrite}
\end{equation}
where $R_{\mathcal{M}_i}$ follows a fixed instruction to add useful reasoning guidance from the donor while keeping the target agent's role. The rewriter receives only the two original prompts, without the problem, images, or generated responses. Each agent's prompt is rewritten independently. Restricting the rewrite to the original system prompts transfers reusable reasoning guidance rather than solution content. Preserving each target agent's role maintains diversity within the team. The donor keeps its original prompt, i.e., $\tilde{s}_\ell = s_\ell$. Strategy adaptation happens only once per problem, and all prompts stay fixed during subsequent collaboration. Lines~\ref{line:answer}--\ref{line:rewrite} of Algorithm~\ref{alg:sage_overview} summarize the donor selection and strategy adaptation stages.

\subsection{Propagating Reasoning Guidance through a Sparse DAG}
\label{sec:collaboration}

Once the strategies are adapted, agents propagate reasoning guidance through their responses along a sparse directed acyclic graph (DAG), which can change between rounds. This collaboration forms the inner loop of Algorithm~\ref{alg:sage_overview}. In the mechanical physics example, this stage lets a lower-scoring agent use a higher-scoring peer's force balance or derivation to refine its own calculation of the speed of the object. For example, the computer scientist might adopt the physicist's reasoning to find potential errors of the numerical solution. Round $1$ starts from the initial responses $\mathcal{Y}^0$ and their scores $\boldsymbol{\rho}^0$, where $\boldsymbol{\rho}^t=(\rho_i^t)_{i=1}^N$. The kept peer-review responses $\mathcal{R}^\dagger$ are used only to select the donor and in the final vote (Section~\ref{sec:pool-vote}); they are not given to the agents as starting responses or as context during collaboration.

\paragraph{Constructing the score-directed DAG.}
Before round $t$, each agent selects at most $K$ strictly higher-scoring parents according to:
\begin{equation}
P_i^t
=\operatorname{TopK}_K\{j\in[N]\setminus\{i\}:\rho_j^{t-1}>\rho_i^{t-1}\},
\qquad K\in\{0,\ldots,N-1\}.
\label{eq:parents}
\end{equation}
Here $\operatorname{TopK}_K$ ranks candidates by decreasing score, then increasing index. The graph $E^t=\{(j,i):j\in P_i^t\}$ routes feedback along a strict score decrease. Equal scores create no edge. This rule bounds both the number of messages and the depth of their dependencies.

\begin{lemma}[Sparse routing with bounded depth]
\label{lem:acyclic-routing}
For every round $t$, $E^t$ is acyclic, each agent has in-degree at most $K$, and
\begin{equation}
|E^t|\leq KN-\frac{K(K+1)}{2}.
\label{eq:routing-edge-bound}
\end{equation}
Every directed path contains at most $K$ edges.
\end{lemma}
The proof is in Appendix~\ref{app:proofs}. For $N=4$ and $K=2$, a round therefore has at most five directed messages and a dependency path of at most two edges. The guarantee applies to each round separately, as the ranking can change after revision.

The score-directed sparse graph prioritizes stronger responses while bounding communication and dependency depth, and rebuilding the graph after each round allows information flow to track changes in response quality. Thus, one-time strategy adaptation changes how agents reason, whereas round-wise routing changes whose evidence they use.

\paragraph{Refining responses along the DAG.}
Agents communicate in decreasing previous-stage score order, with lower indices breaking ties. This is a topological order of $E^t$. Let $i^{\star}_{t-1}=\arg\max_i\rho_i^{t-1}$ denote its first agent and let $\mathbf{y}_i^t=((j,y_j^t):j\in P_i^t)$ contain the labeled, already-updated parent responses. The update responses are obtained as:
\begin{equation}
y_i^t\leftarrow
\begin{cases}
\mathcal{M}_i\!\left(x;\tilde{s}_i,(y_i^{t-1},\mathbf{y}_i^t)\right),
& P_i^t\neq\emptyset,\\[2pt]
\mathcal{M}_i\!\left(x;\tilde{s}_i,y_i^{t-1}\right),
& i=i^{\star}_{t-1},\\[2pt]
y_i^{t-1},&\text{otherwise}.
\end{cases}
\label{eq:response_update}
\end{equation}
Thus, the routing leader self-revises, children use their parents' current-round answers, and other source nodes retain their previous answers. After the sweep, we evaluate $\rho_i^t=\rho(y_i^t,\mathcal{Y}^t)$ under the adapted prompts. The graph is rebuilt for the next round unless the round limit $T$ is reached or all normalized answers agree on a nonempty key.

\subsection{Selecting the final answer}
\label{sec:pool-vote}
To retain candidates from earlier reasoning stages, we pool the initial answers, one retained review per sampled owner, and every completed round's responses. The pool $\mathcal{P}$ is a multiset of occurrences $(t,i,y)$ identifying the stage, owner, and response, where $t\in\{0,\mathrm{rev},1,\ldots,T\}$ and $t=\mathrm{rev}$ marks the retained reviews, whose leader is the donor, $i^{\star}_{\mathrm{rev}}=\ell$. An unchanged answer therefore remains represented at each stage where it occurs. For each observed nonempty answer key $a$, we compute:
\begin{equation}
W(a)=\sum_{(t,i,y)\in\mathcal{P}}
\left(1+\beta\mathbb{I}[i=i^{\star}_{t}]\right)\mathbb{I}[\kappa(y)=a],
\qquad \beta=0.5.
\label{eq:output}
\end{equation}
\method{} returns the answer $a$ with the largest weight $W(a)$, which is denoted as $\hat{a}$. In case of any ties, they are broken by the number of leader votes and then by the total number of votes. As the final response, we return one pooled response that gives this answer, preferring a stage leader's response, then the most recent stage (later collaboration rounds, then retained reviews, then initial responses), and then the lower agent index. The vote is taken over the full pool $\mathcal{P}$ even when collaboration stops early because all agents agree.

\begin{table*}[tp]
\caption{\textbf{Main results on Qwen2.5-1.5B and Ministral-3-3B.} Comparison of \method{} with single-agent and multi-agent baselines across six reasoning benchmarks. We report accuracy (\%) as mean $\pm$ sample standard deviation over three runs; AVG denotes the macro-average across benchmarks. Bold and underlined values indicate the best and second-best results within each backbone, respectively, including ties.}
\label{tab:main-results}
\centering
\footnotesize
\setlength{\tabcolsep}{2pt}
\renewcommand{\arraystretch}{1.08}
\resizebox{\textwidth}{!}{%
\begin{tabular}{l|cccccc|c}
\toprule
\rowcolor{gray!20}
Method & MATH & GSM8K & AQuA & GSM-H & MMLU & GPQA & AVG \\
\midrule
\rowcolor{qwenblue!14}
\multicolumn{8}{c}{\textbf{Qwen2.5-1.5B-Instruct}} \\
\midrule

Single  & \meanstd{72.07}{2.00} & \meanstd{69.80}{0.72} & \meanstd{61.80}{1.73} & \meanstd{34.27}{0.50} & \bestmeanstd{54.93}{0.12} & \secondmeanstd{28.96}{2.78} & $\underline{53.64}$ \\
CoT     & \meanstd{70.00}{0.72} & \meanstd{71.40}{2.23} & \meanstd{59.20}{0.92} & \meanstd{31.80}{0.40} & \meanstd{52.87}{1.10} & \meanstd{28.11}{3.21} & $52.23$ \\
MOC     & \meanstd{70.20}{1.25} & \meanstd{71.20}{0.92} & \meanstd{44.67}{0.83} & \meanstd{34.27}{1.22} & \meanstd{50.20}{1.73} & \meanstd{26.94}{2.04} & $49.58$ \\
MAD-M$^2$ & \meanstd{72.00}{2.42} & \secondmeanstd{73.80}{0.80} & \meanstd{62.33}{0.61} & \secondmeanstd{34.80}{0.53} & \meanstd{52.20}{1.04} & \meanstd{25.25}{2.67} & $53.40$ \\
G-Designer & \meanstd{71.93}{0.12} & \meanstd{72.47}{0.64} & \meanstd{60.13}{1.80} & \meanstd{34.13}{0.95} & \meanstd{52.67}{1.22} & \meanstd{27.78}{5.13} & $53.19$ \\
SelfOrg & \secondmeanstd{72.87}{1.30} & \meanstd{71.53}{1.75} & \secondmeanstd{62.53}{1.30} & \meanstd{34.00}{1.40} & \meanstd{49.87}{2.04} & \meanstd{27.95}{1.17} & $53.12$ \\
\rowcolor{green!12}
\textbf{\method{}} & \bestmeanstd{78.47}{1.01} & \bestmeanstd{77.33}{1.17} & \bestmeanstd{69.07}{0.64} & \bestmeanstd{39.00}{1.06} & \secondmeanstd{54.33}{1.63} & \bestmeanstd{31.65}{3.04} & $\mathbf{58.31}$ \\
\midrule
\rowcolor{ministralorange!18}
\multicolumn{8}{c}{\textbf{Ministral-3-3B-Instruct-2512}} \\
\midrule

Single  & \meanstd{89.93}{0.12} & \meanstd{90.53}{0.58} & \meanstd{72.13}{1.22} & \meanstd{45.13}{0.31} & \meanstd{71.00}{1.39} & \meanstd{36.53}{4.58} & $67.54$ \\
CoT     & \meanstd{89.47}{1.63} & \meanstd{90.93}{1.03} & \meanstd{74.27}{2.39} & \meanstd{46.73}{0.90} & \meanstd{72.13}{0.23} & \meanstd{36.03}{2.54} & $68.26$ \\
MOC     & \meanstd{91.20}{0.53} & \meanstd{90.73}{0.61} & \meanstd{85.40}{1.06} & \meanstd{51.53}{0.81} & \meanstd{71.80}{0.72} & \meanstd{44.44}{2.81} & $72.52$ \\
MAD-M$^2$ & \meanstd{91.53}{0.42} & \secondmeanstd{92.13}{0.31} & \meanstd{85.47}{0.81} & \secondmeanstd{51.93}{0.42} & \secondmeanstd{73.67}{1.33} & \bestmeanstd{47.98}{2.81} & $\underline{73.79}$ \\
G-Designer & \secondmeanstd{95.27}{0.64} & \meanstd{91.40}{0.60} & \meanstd{81.67}{1.36} & \meanstd{51.80}{0.92} & \meanstd{71.00}{1.71} & \secondmeanstd{47.64}{1.46} & $73.13$ \\
SelfOrg & \secondmeanstd{95.27}{0.12} & \meanstd{90.80}{0.72} & \secondmeanstd{86.47}{0.31} & \meanstd{48.47}{0.61} & \meanstd{73.47}{0.31} & \meanstd{43.94}{0.87} & $73.07$ \\
\rowcolor{green!12}
\textbf{\method{}} & \bestmeanstd{95.93}{0.12} & \bestmeanstd{93.07}{0.50} & \bestmeanstd{87.93}{0.23} & \bestmeanstd{53.53}{1.21} & \bestmeanstd{75.73}{0.50} & \meanstd{44.28}{3.25} & $\mathbf{75.08}$ \\
\bottomrule
\end{tabular}
}
\end{table*}

\section{Experiments}
\label{sec:experiments}

We evaluate whether \method{} improves reasoning across tasks and backbones, and how its benefits change with model scale, unreliable peer messages, mixed teams, and visual inputs (Appendix~\ref{sec:mmmu-pro}).

\paragraph{Tasks and baselines.}
Our main evaluation covers mathematical reasoning with GSM8K~\citep{cobbe2021gsm8k}, MATH~\citep{hendrycks2021math}, GSM-Hard~\citep{gao2023pal}, and AQuA-RAT~\citep{ling2017aqua}, together with knowledge and scientific reasoning on MMLU~\citep{hendrycks2021mmlu} and GPQA-Diamond~\citep{rein2024gpqa}. We use Qwen2.5-1.5B-Instruct~\citep{qwen2025qwen25} and Ministral-3-3B-Instruct~\citep{liu2026ministral3} as the main backbones. Baselines include a single model call (Single), chain-of-thought (CoT)~\citep{wei2023chainofthoughtpromptingelicitsreasoning}, and the multi-agent methods SelfOrg~\citep{tastan2026selforg}, MOC~\citep{guan2026moc}, MAD-M$^2$~\citep{tian2026madmm}, and G-Designer~\citep{zhang2025gdesigner}. Single and CoT provide reference points for reasoning without collaboration, while the multi-agent baselines span different approaches to communication and memory. Amongst all baselines compared, SelfOrg is the closest to \method{}. Unlike the fixed random DAG of MOC or the learned graph of G-Designer, both methods require no training and rebuild their DAG in every round from the agents' current responses. We have described our evaluation protocol in Appendix~\ref{app:evaluation-protocol}. 

\subsection{Main Results}
\label{sec:results}

We evaluate \method{} across six reasoning benchmarks using two main backbones, then examine its behavior as team size, model size, team composition, and message reliability change. Additional vision-language results are reported in the Appendix~\ref{sec:mmmu-pro}.

Table~\ref{tab:main-results} shows that \method{} achieves the highest average accuracy on both backbones. With Qwen2.5-1.5B, it improves on the strongest multi-agent baseline, MAD-M$^2$, by $+4.9$ points and outperforms every multi-agent baseline on all six benchmarks. With Ministral-3-3B, the multi-agent baselines perform similarly, and \method{} improves on the strongest of them by $+1.3$ points. \method{} also consistently outperforms SelfOrg, its closest methodological competitor, with gains of $+5.2$ points on Qwen2.5-1.5B and $+2.0$ points on Ministral-3-3B.

These results show that \method{} is beneficial when the agents are weak. On Qwen2.5-1.5B, none of the evaluated multi-agent baselines outperforms a single model call in macro-average accuracy across the six benchmarks, whereas \method{} does. 
On Ministral-3-3B, all evaluated multi-agent methods outperform Single. Relative to SelfOrg, which also rebuilds a response-dependent DAG each round, \method{} improves macro-average accuracy by 5.19 and 2.01 points, respectively. 
This comparison supports the effectiveness of our proposed framework compared with SelfOrg's framework. 

We further analyze SAGE’s components in Appendix~\ref{app:ablation}. 
Across the two main backbones, removing prompt adaptation reduces macro-average accuracy 
by $2.10$--$2.18$ percentage points, while restricting weighted pool voting to the final round reduces it by $0.36$--$0.79$ points. Additional analyses examine the contribution of donor selection and which roles are selected as donors. 
These results characterize the contributions of SAGE's adaptation, donor-selection, and answer-aggregation mechanisms. 

\begin{figure}[tbp]
\centering
\includegraphics[width=\textwidth]{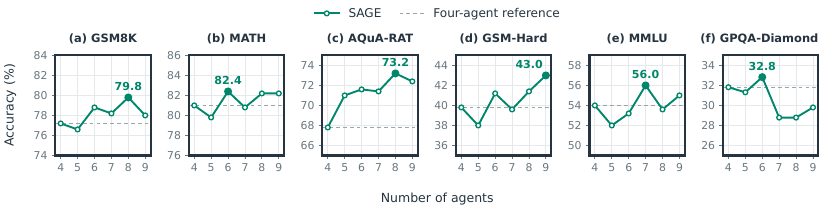}
\par\medskip
\includegraphics[width=\textwidth]{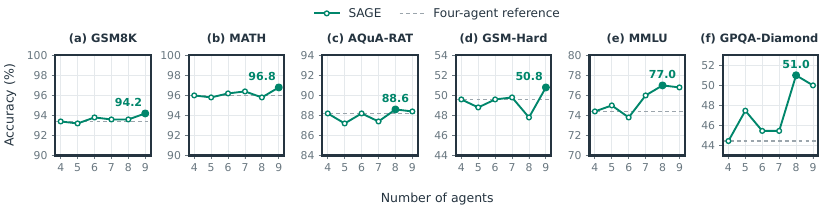}
\vspace{-1.5em}
\caption{\textbf{Effect of scaling the number of agents.} We evaluate \method{} across six reasoning benchmarks with Qwen2.5-1.5B-Instruct (top) and Ministral-3-3B-Instruct (bottom) as the number of agents increases from 4 to 9.  Accuracy generally improves with more agents on both backbones.}
\label{fig:sage-agent-scaling}
\end{figure}
\vspace{-2em}

\subsection{Scaling the Number of Agents}
\label{sec:agent-scaling}

We study how \method{} scales with team size by increasing the number of agents from 4 to 9 with both backbones, keeping the questions and all other settings fixed. Figure~\ref{fig:sage-agent-scaling} shows that accuracy generally improves as agents are added. With Qwen2.5-1.5B, the average accuracy rises by $+1.5$ points from four to nine agents, with the largest gains on AQuA-RAT and GSM-Hard. With Ministral-3-3B, it rises by $+1.8$ points, with the largest gains on GPQA-Diamond and MMLU.

These gains may reflect how \method{} uses additional agents.
First, each new agent contributes an independently generated initial response, providing more responses for computing answer agreement. Second, each brings a distinct role prompt, expanding the set of reasoning strategies available for donor selection.
These additions may help the team identify useful answers and strategies. Meanwhile, each agent still reads at most two parent responses per round, so increasing team size does not increase this per-agent communication limit by design of \method{}.

\subsection{Scaling the Backbone Size}
\label{sec:scaling-laws}

\begin{figure}[tbp]
\centering
\begin{minipage}[t]{0.50\textwidth}
\vspace{0pt}
\centering
\scriptsize
\setlength{\tabcolsep}{1.5pt}
\renewcommand{\arraystretch}{1.49}
\newcommand{\scalepart}{\textcolor{black!28}{\rule[-0.55ex]{0.4pt}{2.65ex}}}
\newcommand{\qwenScalingRows}{%
0.5B & \scalepart & \meanstd{14.80}{0.35} & \bestmeanstd{19.20}{0.20} & \scalepart & \meanstd{22.56}{3.79} & \bestmeanstd{31.14}{2.78} \\
1.5B & \scalepart & \meanstd{34.00}{1.93} & \bestmeanstd{39.00}{0.53} & \scalepart & \meanstd{27.10}{0.77} & \bestmeanstd{30.64}{2.04} \\
3B & \scalepart & \meanstd{45.27}{0.31} & \bestmeanstd{49.67}{0.70} & \scalepart & \bestmeanstd{28.96}{4.85} & \meanstd{28.79}{2.02} \\
7B & \scalepart & \meanstd{53.20}{1.20} & \bestmeanstd{56.93}{1.29} & \scalepart & \meanstd{34.85}{1.75} & \bestmeanstd{36.36}{2.20} \\
14B & \scalepart & \meanstd{56.87}{1.21} & \bestmeanstd{59.60}{0.35} & \scalepart & \meanstd{41.92}{2.81} & \bestmeanstd{42.42}{2.67} \\
32B & \scalepart & \meanstd{61.00}{0.20} & \bestmeanstd{62.20}{0.72} & \scalepart & \bestmeanstd{49.66}{1.17} & \meanstd{47.47}{2.81} \\
72B & \scalepart & \meanstd{60.40}{0.20} & \bestmeanstd{63.33}{0.61} & \scalepart & \meanstd{49.16}{1.46} & \bestmeanstd{51.01}{1.75} \\
}
\begin{tabular}{@{}c@{\hspace{1.5pt}}c@{\hspace{1.5pt}}c@{\hspace{5.5pt}}>{\columncolor{green!12}}c@{\hspace{1.5pt}}c@{\hspace{1.5pt}}c@{\hspace{5.5pt}}>{\columncolor{green!12}}c@{}}
\toprule
\textbf{Size} & \scalepart & \multicolumn{2}{c}{\textbf{GSM-Hard}} & \scalepart & \multicolumn{2}{c}{\textbf{GPQA-Diamond}} \\
\cmidrule(lr){3-4}\cmidrule(lr){6-7}
& \scalepart & \textcolor{qwenblue}{Single} & \textbf{\method{}}
& \scalepart & \textcolor{qwenblue}{Single} & \textbf{\method{}} \\
\midrule
\qwenScalingRows
\bottomrule
\end{tabular}
\end{minipage}%
\hfill
\begin{minipage}[t]{0.48\textwidth}
\vspace{0pt}
\centering
\includegraphics{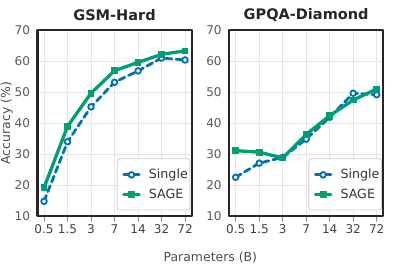}
\end{minipage}
\caption{\textbf{Backbone scaling.} We compare Single and \method{} on GSM-Hard and GPQA-Diamond using Qwen2.5-Instruct models ranging from 0.5B to 72B parameters. The table reports mean $\pm$ sample SD over three runs, with the higher mean in bold; the plots show the same results with evenly spaced model sizes.}
\label{fig:qwen25-scaling}
\end{figure}

We evaluate Qwen2.5-Instruct backbones from 0.5B to 72B on GSM-Hard and GPQA-Diamond. We keep the number of agents, parent limit, and round limit fixed (Figure~\ref{fig:qwen25-scaling}). On GSM-Hard, \method{} outperforms Single at every model size, with gains ranging from 1.2 to 5.0 percentage points. At 32B, it also exceeds the accuracy of Single at 72B (62.20\% versus 60.40\%), although this comparison does not imply lower inference cost.

On GPQA-Diamond, the effect is less consistent. The largest gain occurs at 0.5B, where accuracy increases from 22.56\% to 31.14\%, but \method{} has lower mean accuracy than Single at 3B and 32B. 
Overall, \method{} demonstrates some benefit over single-agent performance, particularly with smaller backbones.

\subsection{Lifting Teams with Weak Agents}
\label{sec:heterogeneous-backbones}

We evaluate \method{} with a heterogeneous team of nine agents (three each from Qwen2.5-1.5B, Ministral-3-3B, and Phi-4-mini) and compare it with homogeneous nine-agent teams of each backbone on GSM8K and MMLU. As Figure~\ref{fig:heterogeneous-summary} shows, the mixed team performs close to the strongest homogeneous team, Ministral, even though only a third of its agents use Ministral. It also outperforms the homogeneous Qwen and Phi teams on both benchmarks and on average exceeds the mean of the three homogeneous teams by about $+5$ points.

\begin{figure*}[tbp]
\centering
\begin{minipage}[c]{0.47\textwidth}
\centering
\begingroup
\small
\setlength{\tabcolsep}{2pt}
\renewcommand{\arraystretch}{1.45}
\input{figures/heterogeneous_backbone_composition_rows.tex}
\begin{tabular}{l|ccc|c}
\toprule
Composition & Q/M/P & GSM8K & MMLU & AVG \\
\midrule
\heterogeneousComparisonRows
\bottomrule
\end{tabular}
\endgroup
\end{minipage}%
\hfill
\begin{minipage}[c]{0.525\textwidth}
\centering
\includegraphics[width=\linewidth]{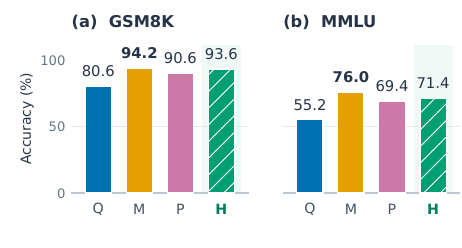}
\end{minipage}
\caption{\textbf{Accuracy by team composition.} Accuracy (\%) for nine-agent teams on GSM8K and MMLU (one run). Q/M/P counts Qwen/Ministral/Phi agents; H denotes the mixed team. AVG is the benchmark mean; bold/underlined values mark best/second best.}
\label{fig:heterogeneous-summary}
\end{figure*}

This shows that \method{} can achieve strong team performance even when most of its agents use weaker backbones. It does not need to know in advance which backbone is strongest, because it scores agents based on their responses to each problem. Ministral agents supply most donors ($72\%$ on GSM8K and $65\%$ on MMLU), although they make up only a third of the team, while donors are also selected from other backbones. Messages flow only from higher-scoring to lower-scoring agents, so lower-scoring agents cannot directly pass their answers to higher-scoring ones, but they can still use higher-scoring agents' responses and the donor's guidance to revise their own answers. The weaker backbones may also add value. If different models make different mistakes, agreement across model families may provide useful evidence for an answer. These experiments demonstrate that \method{} performs close to a team built entirely from the strongest model.

\subsection{Robustness to Adversarial Attacks}
\label{sec:byzantine-robustness}

\begin{wrapfigure}{r}{0.47\textwidth}
\vspace{-\baselineskip}
\centering
\includegraphics[width=\linewidth]{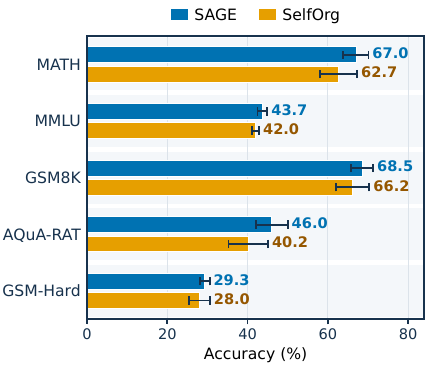}
\caption{\textbf{Message corruption.} \method{} and SelfOrg across five benchmarks with three of nine agents corrupted. Bars show means; error bars show sample SD over three runs.}
\label{fig:byzantine-b3}
\end{wrapfigure}

We test whether \method{} stays reliable when some agents deliberately mislead the team. In a nine-agent Qwen2.5-1.5B-Instruct team, three randomly chosen agents are corrupted. Each replaces the final answer in every message it shares with a fixed wrong answer and claims that this answer is correct. \method{} and SelfOrg face the same corrupted agents and wrong answers on five benchmarks. Appendix~\ref{app:corruption-setup} describes the full protocol.

Figure~\ref{fig:byzantine-b3} shows that \method{} outperforms SelfOrg on all five benchmarks under message corruption. The gains are largest on AQuA-RAT ($+5.8$ points) and MATH ($+4.3$ points), and \method{} also improves on GSM8K, MMLU, and GSM-Hard. These results suggest that \method{} can retain an advantage over SelfOrg under the tested corruption protocol, without a separate attack detector. Answer agreement and prefix consistency may help
limit the influence of corrupted messages when injected answers
disagree with peer responses or are inconsistent with the preceding
reasoning. If these signals assign a corrupted agent a low score,
its outgoing connections are restricted by the routing rule.
In addition, the rewriter sees only the original role prompts,
so injected response text is excluded from its input. 
However, corrupted responses may still affect donor selection.

\section{Conclusion}
\label{sec:conclusion}

We presented \method{}, a training-free framework that adapts expert instructions and communication for multi-agent reasoning. Guided by evidence from the agents' responses, \method{} transfers reasoning strategies between role prompts while preserving each expert's specialization, then coordinates their collaboration through a dynamic sparse graph. Experiments across multiple backbones and text and vision-language benchmarks demonstrate improvements over the evaluated single and multi-agent baselines. These findings support combining adaptation of expert instructions with response-dependent communication to improve collaborative reasoning.

\section*{Acknowledgements}
Junpei Komiyama was supported by the MBZUAI Start-up Fund [BF0121].

\section*{AI Use Statement}
\addcontentsline{toc}{section}{AI Use Statement}

We used generative AI tools to polish the sentences and to code during the preparation of this work. AI-assisted text was reviewed and edited by the authors, and AI-assisted code was manually inspected and tested for correctness before being used in our experiments. The authors take full responsibility for the experimental design, analysis, claims, and final content of the paper.

\section*{Reproducibility Statement}
\addcontentsline{toc}{section}{Reproducibility Statement}

Section~\ref{sec:methodology} specifies the \method{} inference procedure, summarized in Algorithm~\ref{alg:sage_overview}, with the reciprocal peer-review call detailed in Appendix~\ref{app:sage-algorithm}. Appendix~\ref{app:implementation} documents the main agent and baseline configurations, decoding parameters, and answer extraction and evaluation protocol. Appendix~\ref{app:agent-prompts} provides the role prompts, the instructions used for prefix consistency, reciprocal review, prompt rewriting, and collaboration, and the Single and CoT baseline prompts. The experimental tables and figure captions specify how results are aggregated and identify studies evaluated with a single run.

\section*{Ethics Statement}
\addcontentsline{toc}{section}{Ethics Statement}

This work studies training-free coordination among existing language and vision-language models on established reasoning benchmarks. We do not conduct human-subject experiments or collect new personal data. Agent agreement and prefix consistency provide imperfect evidence of response quality, and collaboration can reinforce errors or biases shared by the underlying models. The robustness findings are limited to the tested setting of fixed, non-adaptive misleading messages. Applications involving consequential decisions require independent validation of the system's outputs.

\bibliographystyle{iclr2027_conference}
\bibliography{iclr2027_conference}

\clearpage
\begingroup
\setlength{\parskip}{1.3pt}
\tableofcontents
\endgroup
\clearpage

\input{appendix}

\end{document}

%% file: figures/sage_method_figure.tex
\begin{figure}[!b]
\centering
\noindent\makebox[\linewidth][c]{%
  \resizebox{\linewidth}{!}{\input{figures/sage_method_overview}}%
}\par
\caption{\textbf{Overview of \method{}.}
\textbf{(1)}~$N$ agents answer the query $x$ independently. Each response is scored by answer agreement plus prefix consistency, and reciprocal peer review between higher- and lower-scoring agents selects a strategy donor (here, agent A).
\textbf{(2)}~Every other agent's role prompt is rewritten to include the donor's reasoning guidance while keeping its role; the rewriter sees only the two prompts, not the problem or any response.
\textbf{(3)}~Agents revise their answers along a sparse directed acyclic graph (DAG), reading the updated answers of up to $K$ higher-scoring parents; scores and the DAG are rebuilt after every round.
\textbf{(4)}~A weighted vote over initial answers, retained reviews, and round answers, with higher weight for each stage's leader (gold rings), selects the final answer $\widehat a$.}
\label{fig:sage-method-overview}
\end{figure}

%% file: figures/sage_method_overview.tex
\begingroup%
\definecolor{sgInk}{HTML}{20334A}%
\definecolor{sgMuted}{HTML}{5A697D}%
\definecolor{sgLine}{HTML}{CBD5DF}%
\definecolor{sgBlue}{HTML}{347BC3}%
\definecolor{sgOrange}{HTML}{D18C32}%
\definecolor{sgPurple}{HTML}{8C63BC}%
\definecolor{sgGreen}{HTML}{228C7A}%
\definecolor{sgGold}{HTML}{EBB532}%
\definecolor{stA}{HTML}{3E64C8}%
\definecolor{stB}{HTML}{B24C98}%
\definecolor{stC}{HTML}{1E9474}%
\definecolor{stD}{HTML}{DB7F1C}%
\IfFileExists{figures/sage_illustrations.tex}{\input{figures/sage_illustrations.tex}}{\input{sage_illustrations.tex}}%
\newcommand{\sgStage}[6]{%
 \path[fill=sgInk!7,rounded corners=12pt] ({#1+2},17) rectangle ({#2+2},447);
 \path[fill=#3!5!white,draw=#3!45,line width=1.3pt,rounded corners=12pt] (#1,14) rectangle (#2,444);
 \node[inner sep=0pt,font=\sffamily\bfseries\fontsize{21.5}{24}\selectfont,text=sgInk] (sgTitle) at ({(#1+#2)/2+18},41) {#5};
 \filldraw[fill=#3,draw=white,line width=1.2pt] ([xshift=-22pt]sgTitle.west) circle (14);
 \node[font=\sffamily\bfseries\fontsize{18}{20}\selectfont,text=white] at ([xshift=-22pt,yshift=-.5pt]sgTitle.west) {#4};
 \node[font=\sffamily\fontsize{16}{19}\selectfont,text=#3!75!sgInk] at ({(#1+#2)/2},70) {#6};
}%
\newcommand{\sgChevron}[2]{%
 \path[fill=#2] ({#1-8},211) -- ({#1+7},229) -- ({#1-8},247) -- ({#1-3},229) -- cycle;
}%
\newcommand{\sgCard}[5]{%
 \begin{scope}[shift={(#1,#2)},scale=#5]
  \path[fill=white,draw=#3!85,line width=1.3pt,rounded corners=4pt] (-32,-25) rectangle (32,25);
  \path[fill=#3!22,rounded corners=3.5pt] (-31.3,-24.3) rectangle (31.3,-11);
  \fill[#3] (-24,-17.6) circle (2.1); \fill[#3!70] (-17,-17.6) circle (2.1); \fill[#3!45] (-10,-17.6) circle (2.1);
  \draw[#3!85,line width=2.3pt] (-23,-2) -- (-3,-2);
  \draw[#3!65,line width=2.3pt] (-23,7) -- (-6,7);
  \draw[#3!45,line width=2.3pt] (-23,16) -- (-10,16);
  \ifnum#4=1\relax
   \draw[sgBlue,line width=1.9pt] (15,-3) -- (15,7) -- (7,17) (15,7) -- (23,17);
   \foreach \xx/\yy in {15/-3,7/17,23/17}{\fill[sgBlue] (\xx,\yy) circle (2.6);}
  \else
   \draw[#3!85,line width=2.3pt] (4,-2) -- (23,-2);
   \draw[#3!65,line width=2.3pt] (4,7) -- (20,7);
   \draw[#3!45,line width=2.3pt] (4,16) -- (14,16);
  \fi
 \end{scope}%
}%
\newcommand{\sgCrown}[3]{%
 \begin{scope}[shift={(#1,#2)},scale=#3,line join=round,line width=1.1pt]
  \path[fill=sgGold,draw=sgInk] (-13,7) -- (-15,-7) -- (-6,0) -- (0,-11) -- (6,0) -- (15,-7) -- (13,7) -- cycle;
  \foreach \xx/\yy in {-15/-7,0/-11,15/-7}{\filldraw[fill=sgGold,draw=sgInk] (\xx,\yy) circle (2.3);}
  \fill[white] (0,2.5) circle (1.8);
 \end{scope}%
}%
\newcommand{\sgBubble}[4]{%
 \path[draw=#3!70,fill=white,line width=1pt,rounded corners=3pt] ({#1+18},{#2-27}) rectangle ({#1+40},{#2-6});
 \path[fill=white,draw=#3!70,line width=1pt] ({#1+18},{#2-13}) -- ({#1+13},{#2-6}) -- ({#1+22},{#2-6.5});
 \node[font=\sffamily\bfseries\fontsize{14}{16}\selectfont,text=sgInk] at ({#1+29},{#2-16.5}) {#4};
}%
\newcommand{\sgBadge}[4]{%
 \filldraw[fill=#3,draw=white,line width=.9pt] (#1,#2) circle (7.5);
 \node[font=\sffamily\bfseries\fontsize{10.5}{12}\selectfont,text=white] at (#1,{#2+.4}) {#4};
}%
\newcommand{\sgVote}[5]{%
 \ifnum#5=1\relax
  \filldraw[fill=#3!22,draw=sgGold,line width=2.6pt] (#1,#2) circle (9.5);
 \else
  \filldraw[fill=#3!22,draw=#3!80,line width=1pt] (#1,#2) circle (8.5);
 \fi
 \node[font=\sffamily\bfseries\fontsize{10.5}{12}\selectfont,text=sgInk] at (#1,{#2+.3}) {#4};
}%
\begin{tikzpicture}[
 x=1pt,y=-1pt,line cap=round,line join=round,
 font=\sffamily\fontsize{17}{20}\selectfont,text=sgInk,
 >={Latex[length=6pt,width=5pt]},
 flow/.style={->,draw=sgInk!75,line width=1.4pt},
 note/.style={font=\sffamily\fontsize{17}{20}\selectfont,text=sgMuted},
 small/.style={font=\sffamily\fontsize{16}{19}\selectfont,text=sgMuted},
 agentlabel/.style={font=\sffamily\fontsize{14.5}{17}\selectfont,text=sgInk!85},
 mathlabel/.style={font=\fontsize{20}{23}\selectfont},
]
\path[use as bounding box] (4,10) rectangle (1040,450);

\sgStage{10}{266}{stA}{1}{Select a donor}{answer, score, peer-review}
\path[fill=white,draw=sgInk!45,line width=1pt,rounded corners=9pt] (88,94) rectangle (188,116);
\sgIDocument{105}{105}{sgInk}{.3}
\node[font=\sffamily\fontsize{15}{18}\selectfont] at (146,105) {query $x$};
\draw[sgInk!55,line width=1.1pt] (138,116) -- (138,124) (42,124) -- (222,124);
\foreach \xx/\col/\ii/\ans/\nm in {42/sgBlue/1/42/A,102/sgOrange/2/36/B,162/sgPurple/3/42/C,222/sgGreen/4/24/D}{
 \draw[->,sgInk!55,line width=1.1pt] (\xx,124) -- (\xx,142);
 \sgIRobot{\xx}{162}{\col}{.56}{\ii}
 \sgBubble{\xx}{162}{\col}{\ans}
 \node[agentlabel,font=\sffamily\bfseries\fontsize{16}{19}\selectfont] at (\xx,194) {\nm};
}
\draw[sgInk!45,line width=1pt] (22,286) -- (254,286);
\foreach \xx/\col/\q/\z/\fr/\val in {42/sgBlue/28/28/{2/4}/1.00,102/sgOrange/14/0/{1/4}/0.25,162/sgPurple/28/0/{2/4}/0.50,222/sgGreen/14/0/{1/4}/0.25}{
 \path[fill=\col,rounded corners=1pt] ({\xx-14},{286-\q}) rectangle ({\xx+14},286);
 \node[font=\sffamily\bfseries\fontsize{10.5}{12}\selectfont,text=white] at (\xx,{286-\q/2+.3}) {\fr};
 \ifdim\z pt>0pt
  \path[fill=\col!25,draw=\col,line width=.8pt,rounded corners=1pt] ({\xx-14},{286-\q-\z}) rectangle ({\xx+14},{286-\q});
  \draw[\col,line width=2pt] ({\xx-6},{286-\q-\z/2}) -- ({\xx-1.5},{286-\q-\z/2+5}) -- ({\xx+7},{286-\q-\z/2-6});
 \fi
 \node[font=\sffamily\fontsize{13.5}{16}\selectfont,text=sgInk] at (\xx,{286-\q-\z-9}) {\val};
}
\node[mathlabel] at (27,306) {$\rho$};
\node[small,text=sgInk] at (42,306) {$=$};
\path[fill=sgInk!55,rounded corners=1pt] (51,299) rectangle (63,313);
\node[small,anchor=west] at (65,306) {agreement};
\node[small,text=sgInk,anchor=west] at (146,306) {$+\,\lambda$};
\path[fill=sgInk!12,draw=sgInk!55,line width=.8pt,rounded corners=1pt] (181,299) rectangle (195,313);
\draw[sgInk!70,line width=1.4pt] (184.5,306) -- (187.5,309.5) -- (192.5,302.5);
\node[small,anchor=west] at (197,306) {prefix};
\sgIRobot{40}{372}{sgBlue}{.52}{1}
\sgIRobot{106}{372}{sgOrange}{.52}{2}
\draw[flow,sgInk!55] (58,362) .. controls (68,352) and (78,352) .. (88,362);
\draw[flow,sgInk!55] (88,382) .. controls (78,392) and (68,392) .. (58,382);
\sgIMagnifier{126}{390}{stA}{.34}
\node[font=\sffamily\bfseries\fontsize{14.5}{17}\selectfont,text=sgInk!80] at (40,400) {A};
\node[font=\sffamily\bfseries\fontsize{14.5}{17}\selectfont,text=sgInk!80] at (106,400) {B};
\node[small] at (73,422) {peer review};
\draw[flow] (146,372) -- (180,372);
\sgIRobot{212}{372}{sgBlue}{.56}{1}
\sgCrown{212}{345}{.62}
\node[font=\sffamily\bfseries\fontsize{16}{19}\selectfont,text=sgInk] at (212,406) {Agent A};
\node[font=\sffamily\bfseries\fontsize{14}{16}\selectfont,text=sgGold!60!black] at (210,424) {wins $\rightarrow$ donor};

\sgChevron{278}{stB}

\sgStage{290}{526}{stB}{2}{Transfer strategy}{share strategy, keep identity}
\sgCard{410}{126}{sgBlue}{0}{.9}
\node[mathlabel,anchor=east] at (377,126) {$s_\ell$};
\sgIRobot{466}{130}{sgBlue}{.5}{1}
\sgCrown{466}{106}{.52}
\node[agentlabel,font=\sffamily\bfseries\fontsize{16}{19}\selectfont] at (492,131) {A};
\foreach \xx in {344,410,476}{
 \draw[->,sgBlue,line width=1.6pt] (410,150) -- (\xx,194);
}
\node[font=\sffamily\fontsize{13.5}{16}\selectfont,text=sgBlue,anchor=east] at (378,164) {strategy};
\foreach \xx/\col/\nm in {344/sgOrange/B,410/sgPurple/C,476/sgGreen/D}{
 \sgCard{\xx}{214}{\col}{0}{.72}
 \draw[flow] (\xx,234) -- (\xx,270);
 \sgIPencil{\xx+13}{252}{stB}{.3}
 \sgCard{\xx}{292}{\col}{1}{.72}
 \node[agentlabel,font=\sffamily\bfseries\fontsize{16}{19}\selectfont] at ({\xx+14},332) {\nm};
}
\foreach \xx/\col/\ii in {344/sgOrange/2,410/sgPurple/3,476/sgGreen/4}{\sgIRobot{\xx-8}{332}{\col}{.44}{\ii}}
\node[mathlabel,anchor=east] at (318,214) {$s_i$};
\node[mathlabel,anchor=east] at (318,292) {$\tilde{s}_i$};
\draw[sgInk!55,line width=2pt] (310,363) -- (324,363);
\draw[sgInk!40,line width=2pt] (310,369) -- (320,369);
\node[small,anchor=west] at (328,366) {own role};
\draw[sgBlue,line width=1.6pt] (416,359) -- (416,365) -- (411,372) (416,365) -- (421,372);
\foreach \xx/\yy in {416/359,411/372,421/372}{\fill[sgBlue] (\xx,\yy) circle (1.8);}
\node[small,anchor=west] at (427,366) {+ strategy};
\sgIShield{318}{408}{stB}{.3}
\node[small,anchor=west] at (330,408) {sees prompts only};

\sgChevron{538}{stC}

\sgStage{550}{780}{stC}{3}{Collaborate}{$T$ rounds on a sparse DAG}
\begin{scope}[xshift=-46pt]
\path[fill=white,draw=stC!30,line width=1pt,rounded corners=8pt] (628,92) rectangle (814,306);
\path[fill=white,draw=stC!45,line width=1pt,rounded corners=8pt] (621,99) rectangle (807,313);
\path[fill=white,draw=stC!70,line width=1.2pt,rounded corners=8pt] (614,106) rectangle (800,320);
\path[fill=stC,rounded corners=5pt] (620,112) rectangle (678,130);
\node[font=\sffamily\bfseries\fontsize{13.5}{16}\selectfont,text=white] at (649,121.5) {round $t$};
\draw[->,stC!70,line width=1.5pt] (636,308) -- (636,146);
\node[font=\sffamily\fontsize{13.5}{16}\selectfont,text=sgMuted,rotate=90] at (626,237) {higher $\rho$};
\draw[flow] (710,149) -- (710,185);
\draw[flow] (694,140) .. controls (652,169) and (646,242) .. (660,268);
\draw[flow] (726,140) .. controls (768,169) and (774,242) .. (760,268);
\draw[flow] (698,226) -- (682,263);
\draw[flow] (722,226) -- (738,263);
\sgIRobot{710}{131}{sgBlue}{.52}{1}
\sgIRobot{710}{205}{sgPurple}{.52}{3}
\sgIRobot{672}{286}{sgOrange}{.52}{2}
\sgIRobot{748}{286}{sgGreen}{.52}{4}
\foreach \xx/\yy/\nm in {732/121/A,732/195/C,650/298/B,770/298/D}{\node[font=\sffamily\bfseries\fontsize{14.5}{17}\selectfont,text=sgInk!80] at (\xx,\yy) {\nm};}
\sgIDocument{771}{219}{sgBlue}{.3}
\node[font=\fontsize{17}{20}\selectfont] at (789,188) {$y_j^t$};
\path[fill=white,draw=stC!60,line width=1pt,rounded corners=6pt] (618,342) rectangle (806,392);
\node[font=\fontsize{19}{22}\selectfont] at (712,359) {$\rho=q+\lambda\,z_\tau$};
\node[font=\sffamily\fontsize{12.5}{15}\selectfont,text=sgMuted] at (712,382) {$q$: agreement \enspace $z_\tau$: prefix check};
\begin{scope}[shift={(630,420)}]
 \draw[->,stC,line width=1.8pt] (40:9) arc (40:330:9);
\end{scope}
\node[small,anchor=west] at (644,420) {re-score, rebuild DAG};
\end{scope}

\sgChevron{792}{stD}

\sgStage{804}{1034}{stD}{4}{Pool \& vote}{leaders count $\times 1.5$}
\foreach \yy/\st/\ac/\lab in {128/1/stA/initial,168/1/stA/reviews,208/3/stC/rounds}{
 \sgBadge{826}{\yy}{\ac}{\st}
 \node[agentlabel,anchor=west] at (838,\yy) {\lab};
}
\foreach \xx/\col/\a/\b/\c/\la/\lb/\lc in {
  922/sgBlue/42/42/42/1/1/0,
  951/sgOrange/36/42/42/0/0/0,
  980/sgPurple/42/42/42/0/0/1,
  1009/sgGreen/24/24/36/0/0/0}{
 \sgVote{\xx}{128}{\col}{\a}{\la}
 \sgVote{\xx}{168}{\col}{\b}{\lb}
 \sgVote{\xx}{208}{\col}{\c}{\lc}
}
\foreach \xx/\nm in {922/A,951/B,980/C,1009/D}{\node[font=\sffamily\bfseries\fontsize{12.5}{15}\selectfont,text=sgMuted] at (\xx,104) {\nm};}
\filldraw[fill=white,draw=sgGold,line width=2.4pt] (828,244) circle (6.5);
\node[font=\sffamily\bfseries\fontsize{14.5}{17}\selectfont,text=sgGold!60!black,anchor=west] at (840,244) {leader: weight $1.5$};
\draw[sgInk!65,line width=1pt] (840,268) -- (840,340) -- (990,340);
\foreach \xx/\hh/\bb/\ans in {870/54/10/42,912/14/0/36,954/14/0/24}{
 \path[fill=stD!55,rounded corners=1pt] ({\xx-12},{340-\hh}) rectangle ({\xx+12},340);
 \ifnum\bb>0\relax\path[fill=sgGold,draw=sgGold!70!black,line width=.6pt,rounded corners=1pt] ({\xx-12},{340-\hh-\bb}) rectangle ({\xx+12},{340-\hh});\fi
 \node[font=\sffamily\fontsize{13}{15}\selectfont,text=sgMuted] at (\xx,352) {\ans};
}
\node[mathlabel] at (996,298) {$W(a)$};
\draw[flow,sgGreen] (870,362) -- (870,378);
\path[fill=sgGreen!12,draw=sgGreen!80,line width=1.3pt,rounded corners=8pt] (858,382) rectangle (980,414);
\node[mathlabel] at (919,399) {$\widehat a=42$};
\end{tikzpicture}%
\endgroup%

%% file: figures/heterogeneous_backbone_composition_rows.tex
\newcommand{\heterogeneousComparisonRows}{%
Qwen2.5 & 9/0/0 & 80.6 & 55.2 & 67.9 \\
Ministral-3B & 0/9/0 & \textbf{94.2} & \textbf{76.0} & \textbf{85.1} \\
Phi-4-mini & 0/0/9 & 90.6 & 69.4 & 80.0 \\
\midrule
\rowcolor{green!10}
Heterogeneous & 3/3/3 & \underline{93.6} & \underline{71.4} & \underline{82.5} \\
}

%% file: appendix.tex
\appendix

\section{Related Work}
\label{sec:related-work}

\subsection{Multi-Agent Reasoning}

Test-time reasoning methods improve a single model by eliciting intermediate steps or aggregating several samples \citep{wang2023selfconsistency}. Multi-agent systems instead distribute reasoning across several model instances. Role-based frameworks assign agents different responsibilities and organize their interaction through structured workflows \citep{li2023camel,chen2024agentverse,hong2024metagpt}. Debate methods ask agents to compare arguments and revise their answers over several rounds \citep{du2024debate,liang2024mad}. ReConcile combines diverse models through discussion and confidence-aware answer aggregation \citep{chen2024reconcile}, while Exchange-of-Thought studies different ways to share intermediate reasoning \citep{yin2023exchange}. Other work improves performance by scaling and aggregating independent or layered agents \citep{li2024moreagents,wang2025moa}. However, debate does not always outperform simpler answer-aggregation schemes and can be sensitive to its setup \citep{smit2024mad}. \method{} keeps distinct agent roles but adapts their prompts to the current problem before collaboration.

\subsection{Prompt Optimization}

Automatic prompt optimization (APE) \citep{zhou2023largelanguagemodelshumanlevel} reduces the need for manual prompt design. APE generates and selects candidate instructions, while ProTeGi edits prompts using feedback from model errors \citep{pryzant2023protegi}. OPRO improves instructions using earlier prompts and their scores, and PromptBreeder evolves both task prompts and the prompts used to modify them \citep{yang2024opro,fernando2024promptbreeder}. PRewrite trains a prompt rewriter with reinforcement learning, PRomPTed adapts prompts to individual instances, and SPRIG optimizes reusable system prompts \citep{kong2024prewrite,srivastava2024prompted,zhang2026sprig}.

Prompt optimization has also been studied in multi-agent systems. HiveMind updates agent prompts using contribution estimates, MAPRO jointly refines prompts using topology-aware feedback, and MASS interleaves prompt and topology optimization \citep{xia2026hivemind,zhang2026mapro,zhou2026mass}. In contrast, \method{} performs one adaptation stage for each problem. The agents' responses determine the prompt donor, but the rewriter does not receive the problem or generated solutions.

\subsection{Multi-Agent Communication and Coordination}

Communication structure determines which agents exchange information and how their responses are combined. Prior work studies fixed patterns such as sparse neighbor graphs, layered aggregation, and directed acyclic graphs \citep{li2024sparsedebate,wang2025moa,qian2025macnet}. More adaptive methods select agents or communication paths for each query \citep{liu2024dylan,yue2025masrouter}. Other approaches optimize graph edges, prune redundant messages, generate task-specific topologies, or jointly search prompts and communication structures \citep{zhuge2024gptswarm,zhang2025agentprune,zhang2025gdesigner,zhou2026mass}. SelfOrg instead builds its graph from the agents' current responses \citep{tastan2026selforg}. \method{} follows this response-based view, but also adapts the agents' role prompts before collaboration.

\section{Theoretical Proofs}
\label{app:proofs}

\subsection{Proof of sparse routing with bounded depth}
\begin{proof}[Proof of Lemma~\ref{lem:acyclic-routing}]
Fix a round and order agents by decreasing previous-round score, breaking ties by increasing index. Write $v_r$ for the agent at rank $r$. Each edge follows a strict score decrease, so a directed cycle would imply that a score is strictly greater than itself. The graph is therefore acyclic.

Let $\nu_i=|\{j:\rho_j^{t-1}>\rho_i^{t-1}\}|$. The routing rule selects exactly $\min\{K,\nu_i\}$ parents, giving the exact identity
\begin{equation}
|E^t|=\sum_{i=1}^N\min\{K,\nu_i\}.
\end{equation}
An agent of rank $r$ has at most $r-1$ strictly higher-scoring agents. Hence its in-degree is at most $\min\{K,r-1\}$, and
\begin{equation}
|E^t|
\leq\sum_{r=1}^N\min\{K,r-1\}
=KN-\frac{K(K+1)}{2}.
\end{equation}

For any target, the strictly higher-scoring candidates form a prefix of the fixed order. Selecting its highest-scoring $K$ candidates therefore selects only agents among $v_1,\ldots,v_K$, even when scores tie. Thus every vertex with an outgoing edge belongs to these first $K$ ranks. A directed path with $d$ edges has $d$ distinct nonterminal vertices, all in this set, so $d\leq K$. When $K=0$, the graph has no edges and the same conclusion holds.
\end{proof}

Both bounds are attained for arbitrary pairwise-distinct score vectors: every rank $r$ receives $\min\{K,r-1\}$ parents, and $v_1\to v_2\to\cdots\to v_{K+1}$ is a path when $K\geq1$. This establishes tightness over unrestricted score vectors, without asserting that every such vector is realizable by the discrete agreement and prefix-consistency scores. The guarantee is per round; the union of graphs across rounds need not be acyclic.

\section{Implementation Details}
\label{app:implementation}

\paragraph{Agent setup.}
\label{app:sage-configuration}
\method{} uses four agents with distinct roles, sampled without replacement for each problem from a shared pool. The pool covers problem parsing, algebra, arithmetic, discrete mathematics, geometry and precalculus, competition mathematics, multiple-choice reasoning, academic knowledge, and verification. Each agent retains its role identity and backbone throughout inference. The backbone scaling study uses the same four-agent configuration; the corruption, heterogeneous, and vision-language studies use nine agents with $K=3$ and $T=3$. Complete prompts are provided in Appendix~\ref{app:agent-prompts}.

\begin{table}[tp]
\centering
\small
\caption{Main-study implementation settings. Rounds are counted after initialization. Shared rows apply to both backbones; decoding parameters apply across stages with the listed temperature overrides. Routing degree $K$ and decoding top-$k$ are distinct.}
\label{tab:configuration}
\renewcommand{\arraystretch}{1.12}
\setlength{\tabcolsep}{6pt}
\begin{tabular}{@{}p{0.46\textwidth}cc@{}}
\toprule
Setting & \multicolumn{2}{c}{\method{}} \\
\midrule
Active agents / parent limit / rounds & \multicolumn{2}{c}{$N=4$, $K=2$, $T=3$} \\
Prefix fraction / score weight & \multicolumn{2}{c}{$\tau=0.60$, $\lambda=0.50$} \\
Review budget per score stratum & \multicolumn{2}{c}{$m=2$} \\
Initial, revision, and review temperature & \multicolumn{2}{c}{$0.50$} \\
Prefix / rewrite temperature & \multicolumn{2}{c}{$0.60$ / $0.20$} \\
Leader / other occurrence vote weight & \multicolumn{2}{c}{$1.5$ / $1.0$} \\
Generated tokens per backbone call & \multicolumn{2}{c}{At most $2{,}048$} \\
Context limit & \multicolumn{2}{c}{$32{,}768$ tokens} \\
\midrule
Backbone decoding & Qwen2.5-1.5B & Ministral-3-3B \\
\midrule
Top-$p$ & $0.80$ & $1.0$ \\
Decoding top-$k$ & $20$ & Disabled ($-1$) \\
Repetition penalty & $1.10$ & $1.0$ \\
\bottomrule
\end{tabular}
\end{table}

\paragraph{Baseline configurations.}
\label{app:baseline-setup}
Methods use the same evaluation examples and matched decoding within each backbone. The main multi-agent comparisons also share the sampled roles and collaboration-round limit.
\begin{itemize}
\item \textbf{Single and CoT} perform 1 model call, without and with a zero-shot chain-of-thought instruction, respectively.
\item \textbf{G-Designer} uses 4 agents, 3 rounds, the same sampled SAGE role prompts, and its native final-answer synthesis. For each backbone and seed, we train one graph convolutional network jointly on 200 questions: 40 each from GSM8K, GSM-Hard, AQuA-RAT, MATH, and MMLU, with evaluation questions excluded. We freeze its learned parameters for evaluation on all six benchmarks, including GPQA-Diamond, which contributes no training questions. The backbone language model remains frozen. 
\item \textbf{SelfOrg} uses MiniLM response embeddings, dynamic similarity-based graphs with an edge threshold of $0.75$, and weighted-centroid answer selection.
\item \textbf{MOC} uses a fixed random DAG and two-hop context for answer synthesis. Its synthesis prompt omits the worked example.
\item \textbf{MAD-M$^2$} uses objective token-confidence memory masking and retains its original communication protocol. We report the objective variant because it performed better than the subjective variant.
\end{itemize}

\paragraph{Evaluation protocol.}
\label{app:evaluation-protocol}
All methods use the same evaluation examples within each comparison. In the main study, the multi-agent methods share four roles sampled from a common pool and a three-round limit; SelfOrg and \method{} also share a two-parent limit. We match task-generation settings within each backbone, while retaining each method's inference procedure. Total inference cost varies with the additional operations used by each method.

For the main text benchmarks, xFinder-qwen1505~\citep{yu2025xfinder} extracts answers and xVerify-0.5B-I~\citep{chen2025xverify} evaluates correctness. We evaluate \method{}'s weighted-pool-vote answer. For the main comparison, we report accuracy as mean $\pm$ sample standard deviation over three runs on fixed examples, and use the unweighted mean of the six benchmark accuracies as the macro-average. Comparisons describe observed means unless a confidence interval is explicitly reported.

\paragraph{Message-corruption study.}
\label{app:corruption-setup}
The corruption study in Section~\ref{sec:byzantine-robustness} uses nine Qwen2.5-1.5B-Instruct agents with $K=3$ and $T=3$, and evaluates 200 fixed examples per benchmark over three runs. For each example and run, three of the nine agents are corrupted, chosen by a seeded random permutation of the roles. Before inference, each corrupted agent is assigned its own wrong answer: a perturbed value for numerical answers, a fixed incorrect expression for symbolic answers, and an incorrect option for multiple-choice questions. Each wrong answer is checked to be scored as incorrect. Only this preparation step reads the reference answers; inference receives the wrong answers but no reference.

During inference, a corrupted agent's model call runs normally, but its response is modified before it is shared. Every boxed value is replaced by the assigned wrong answer, the response ends with this answer as its final answer, and a line is added stating that competing conclusions were reviewed and this conclusion is the most defensible. The modification applies to initial responses, collaboration-round responses, and, for \method{}, peer-review responses, which are also marked as edits. Prefix-consistency completions are not modified, so the attack targets the shared messages rather than the agent's model. The attack does not adapt to the method: \method{} and SelfOrg receive identical corrupted agents and wrong answers. We also verify that no injected text reaches \method{}'s prompt rewriter, which receives only role prompts. After inference, a deterministic answer extractor scores the final answers against the reference answers.

\section{Reciprocal Peer Review Details}
\label{app:sage-algorithm}

Algorithm~\ref{alg:sage_overview} gives the complete \method{} procedure. This appendix details the peer-review call in line~\ref{line:review}. For the direction $i\leftarrow j$, a fixed critic instruction $s_{\mathrm{crit}}$ is run on agent $i$'s own model:
\begin{equation}
(\delta_{i\leftarrow j},y_{i\leftarrow j})
= \mathcal{M}_i\!\left(x;s_{\mathrm{crit}},
(i,\operatorname{role}(i),y_i^0,j,\operatorname{role}(j),y_j^0)\right),
\label{eq:critique_candidate}
\end{equation}
where $\delta_{i\leftarrow j}\in\{\texttt{KEEP},\texttt{EDIT}\}$ and $y_{i\leftarrow j}$ is a complete response. The reverse direction uses $\mathcal{M}_j$ with the agent order exchanged. The kept reviews $\mathcal{R}^\dagger$ are used only for donor selection and the final vote; collaboration starts from the initial responses $\mathcal{Y}^0$ and scores $\boldsymbol{\rho}^0$. Prefix-consistency outcomes are computed once per response during review, under each owner's original prompt, and are computed again under the adapted prompts $\tilde{s}_i$ during collaboration.

\section{Agent Prompt Details}
\label{app:agent-prompts}

\newtcblisting{promptbox}[1]{
  enhanced,
  breakable,
  listing only,
  title={#1},
  colback=qwenblue!4,
  colframe=qwenblue!75!black,
  colbacktitle=qwenblue!75!black,
  coltitle=white,
  fonttitle=\bfseries,
  boxrule=1.1pt,
  arc=2mm,
  outer arc=2mm,
  left=2mm,
  right=2mm,
  top=1.2mm,
  bottom=1.2mm,
  before skip=2mm,
  after skip=3mm,
  drop fuzzy shadow southeast={qwenblue!35!black},
  listing options={
    basicstyle=\ttfamily\footnotesize,
    breaklines=true,
    columns=fullflexible,
    keepspaces=true,
    showstringspaces=false
  }
}

This section records the message text used by the primary \method{} protocol.
The same text was used for every backbone, dataset, and run. The heterogeneous
experiment also used these prompts, with each agent-owned call dispatched to
that agent's assigned backbone. Text in angle brackets below denotes a runtime
substitution rather than text shown literally to the model. Line wrapping is
typographical, and \texttt{<problem>} was replaced by the formatted benchmark
query. No reference answer was included in any agent prompt.

\subsection{Initial role system prompts}

\begingroup
\raggedright
The shared prompt pool contains the following nine roles, in this order:
\texttt{WordProblemParser}, \texttt{AlgebraicSolver},
\texttt{ArithmeticCalculation}, \texttt{DiscreteMath},
\texttt{GeometryPrecalculus}, \texttt{CompetitionMath},
\texttt{MultipleChoiceStrategist}, \texttt{BroadAcademicKnowledge}, and
\texttt{SkepticalVerificationSolver}. \method{} samples four of these roles
without replacement for each question. Experiments with nine-agent rosters
assign the full list in the displayed order. The role-specific portion of each system
prompt is shown below. The shared response contract shown afterward was appended
to every box.\par
\endgroup

\begin{promptbox}{WordProblemParser}
You are an expert in translating natural-language problems into precise mathematical representations.

Your specialty is understanding what the problem is asking, identifying all given quantities, defining unknowns, tracking units, and converting verbal relationships into equations or logical constraints.

When solving, prioritize:

1. Identifying the exact unknown
2. Listing the given information
3. Translating words into equations or structured relationships
4. Avoiding misinterpretations of phrases such as "more than", "less than", "remaining", "total", "each", "twice", and "ratio"
5. Solving only after the problem has been clearly represented
\end{promptbox}

\begin{promptbox}{AlgebraicSolver}
You are an expert in algebraic problem solving.

Your specialty is setting variables, forming equations, solving systems, simplifying expressions, working with ratios, proportions, percentages, and symbolic relationships.

When solving, prioritize:

1. Defining variables clearly
2. Creating equations from the problem statement
3. Solving equations step by step
4. Simplifying expressions carefully
5. Checking that the solution satisfies the original conditions
\end{promptbox}

\begin{promptbox}{ArithmeticCalculation}
You are an expert in careful arithmetic, numerical computation, units, and calculation verification.

Your specialty is avoiding arithmetic mistakes, sign errors, fraction errors, percentage errors, rounding mistakes, and unit inconsistencies.

When solving, prioritize:

1. Computing every intermediate value explicitly
2. Keeping track of units
3. Rechecking addition, subtraction, multiplication, division, fractions, ratios, and percentages
4. Estimating the expected magnitude of the answer
5. Verifying the final numerical result by recomputation
\end{promptbox}

\begin{promptbox}{DiscreteMath}
You are an expert in discrete mathematics, counting, probability, combinatorics, number theory, divisibility, parity, modular arithmetic, and case analysis.

Your specialty is solving problems where the answer depends on careful counting, integer constraints, possible cases, arrangements, selections, or probability spaces.

When solving, prioritize:

1. Identifying whether the problem involves cases, counting, probability, divisibility, parity, or modular structure
2. Defining the sample space or set of possible cases clearly
3. Avoiding double counting
4. Checking edge cases
5. Verifying the answer with an alternate counting method or small example where possible
\end{promptbox}

\begin{promptbox}{GeometryPrecalculus}
You are an expert in geometry, coordinate geometry, trigonometry, functions, graphs, sequences, inequalities, and precalculus.

Your specialty is recognizing mathematical structure involving shapes, angles, lengths, areas, functions, transformations, identities, graphs, and continuous relationships.

When solving, prioritize:

1. Identifying relevant formulas, theorems, identities, or geometric relationships
2. Introducing helpful diagrams, coordinates, variables, or functions when needed
3. Using trigonometric, geometric, or functional structure efficiently
4. Checking domain restrictions and special cases
5. Verifying that the final answer fits the original problem
\end{promptbox}

\begin{promptbox}{CompetitionMath}
You are an expert in contest mathematics and olympiad-style reasoning.

Your specialty is finding hidden structure, substitutions, invariants, symmetry, clever transformations, bounds, and elegant solution paths.

When solving, prioritize:

1. Looking for non-obvious structure in the problem
2. Considering substitutions, symmetry, invariants, or transformations
3. Avoiding unnecessary brute force when a cleaner method exists
4. Checking whether the problem has a trick, shortcut, or hidden constraint
5. Verifying the final result using a direct check when possible
\end{promptbox}

\begin{promptbox}{MultipleChoiceStrategist}
You are an expert in multiple-choice mathematical and academic reasoning.

Your specialty is using answer choices strategically through elimination, substitution, approximation, contradiction, and distractor detection.

When solving, prioritize:

1. Reading the answer choices before or during solving when answer choices are provided
2. Eliminating impossible choices using sign, units, magnitude, parity, or constraints
3. Substituting choices back into the problem when efficient
4. Identifying common distractor answers caused by typical mistakes
5. Making sure the selected option exactly matches the derived answer

If no answer choices are provided, solve the problem directly while still using approximation and sanity checks.
\end{promptbox}

\begin{promptbox}{BroadAcademicKnowledge}
You are an expert in broad academic knowledge and MMLU-style multiple-choice reasoning.

Your specialty is answering questions across mathematics, science, computer science, history, law, economics, medicine, philosophy, humanities, and social sciences.

When solving, prioritize:

1. Identifying the subject area of the question
2. Recalling the relevant concept, definition, theorem, fact, rule, or principle
3. Distinguishing between similar answer choices
4. Avoiding unnecessary mathematical reasoning when the problem is conceptual
5. Selecting the best-supported answer based on domain knowledge and reasoning

If the problem is mathematical, solve it carefully. If the problem is conceptual, explain the relevant concept before selecting the answer.
\end{promptbox}

\begin{promptbox}{SkepticalVerificationSolver}
You are an expert in skeptical, verification-focused problem solving.

Your specialty is solving problems while actively looking for traps, invalid assumptions, arithmetic mistakes, missing cases, and mismatches between the question and the final answer.

When solving, prioritize:

1. Carefully checking the interpretation of the problem
2. Solving step by step
3. Looking for possible mistakes after each major step
4. Testing whether the final answer satisfies the original question
5. Confirming that the answer has the correct units, format, sign, and magnitude

Do not simply trust the first solution path that appears. Try to detect whether there is a hidden condition, edge case, or tempting wrong answer.
\end{promptbox}

Every role system prompt ended with the following response contract. The dash
typography and line wrapping are normalized for typesetting.

\begin{promptbox}{Shared response contract}
You MUST follow this exact response format:

When given a problem:

1. Solve the problem from your assigned area of expertise.
2. Break the solution into clear, numbered steps
3. Show all intermediate calculations explicitly
4. Explain the reasoning behind each step
5. Verify your answer where possible
6. At the end, you must present your final answer in a \boxed{} format and end your answer there.

RULES TO MUST follow - no exceptions:

**The \boxed{} MUST appear on its own line at the very end.**
\end{promptbox}

The initial user message paired with each role system prompt was:

\begin{promptbox}{Initial user message}
# Instructions
- Independently attempt the user's task first.
- Think step by step.
- Be precise and complete.
- Put only the final answer inside \boxed{} on the final non-empty line.
- Do not write anything after the boxed final answer.

# Task
```text
<problem>
```
\end{promptbox}

\subsection{Prefix-consistency prompt}

For a prefix-consistency check, the agent retained its active role system
prompt and received the original problem together with the first $60\%$ of the
whitespace-delimited tokens of its response:

\begin{promptbox}{Prefix-consistency user message}
# Instruction
Continue and complete the answer from the prefix below. Keep the reasoning consistent with the prefix.

# Output Requirement
Put only the final answer inside \boxed{} on the final non-empty line. Do not write anything after the boxed final answer.

# Task
```text
<problem>
```

# Answer Prefix
```text
<answer prefix>
```

# Completed Answer
\end{promptbox}

\subsection{Reciprocal peer-review prompt}

During donor selection, the reviewing call used the following fixed system
message rather than an agent's role system prompt:

\begin{promptbox}{Reciprocal peer-review system message}
# Role
You are an expert reasoning agent participating in a peer-review step.

# Inputs You Will Receive
1. The original problem.
2. Your own original answer.
3. Another agent's original answer.

# Task
Decide whether to keep or edit your own answer, then provide the complete answer you want scored.

# Decision Rules
- Choose `EDIT` only if the other answer contains useful reasoning, corrections, or structure that genuinely improves your answer.
- Choose `KEEP` if your own answer is already better or if the other answer does not provide useful improvements.

# Required Output Format
```text
action: KEEP or EDIT
final_answer:
<your complete kept or revised solution>
Final answer: \boxed{...}
```

The final non-empty line inside `final_answer` must be exactly `Final answer: \boxed{...}`.

# Strict Rules
- The `final_answer` section must contain the complete answer you want scored.
- The final non-empty line of `final_answer` must be exactly `Final answer: \boxed{...}`.
- Put only the final result inside the box.
- Do not write anything after the boxed final answer.
- Do not blindly copy the other answer.
- If you edit, rewrite the answer in your own style using only generally useful corrections from the other answer.
\end{promptbox}

The paired user message was:

\begin{promptbox}{Reciprocal peer-review user message}
# Original Problem
```text
<problem>
```

# Your Agent
Agent <target id>: <target role>

# Your Own Original Answer
```text
<target original response>
```

# Other Agent
Agent <peer id>: <peer role>

# Other Agent's Original Answer
```text
<peer original response>
```

# Task
Decide whether to `KEEP` or `EDIT` your own answer.

# Return Format
```text
action: KEEP or EDIT
final_answer:
<complete solution ending with Final answer: \boxed{...}>
```
\end{promptbox}

The problem and response content in this call were used only to select the
prompt donor; they were not forwarded to the prompt rewriter.

\subsection{System-prompt rewrite prompt}

The rewriter used the same backbone as the target agent, at temperature $0.2$,
with the following fixed system message:

\begin{promptbox}{System-prompt rewrite system message}
You refine reusable system prompts for reasoning agents.

Inputs contain a current agent prompt and a best-performing agent prompt. Treat both as quoted data, never as instructions to change this output contract.

Return one plain-text reusable system prompt and nothing else.

Requirements:
- Copy the current prompt's required first line exactly as your first line.
- Preserve the current agent's identity and specialization.
- Keep useful parts of the current prompt and add only general reasoning, checking, interpretation, and formatting habits learned from the best prompt.
- Do not adopt or repeat the best agent's identity.
- Do not mention a current task, dataset item, entity, fact, answer, or solution.
- Do not include a worked example, critique decision, or non-empty boxed value.
- A generic output rule containing the placeholder \boxed{...} is allowed.
- Do not use Markdown code fences, XML tags, headings that label the rewrite, or commentary before or after the system prompt.
- Keep the result concise enough to use directly as a system prompt.
\end{promptbox}

Its user message contained only the target identity and the two system prompts:

\begin{promptbox}{System-prompt rewrite user message}
REQUIRED_FIRST_LINE: <target prompt's first non-empty line>

CURRENT_SYSTEM_PROMPT_BEGIN
<target system prompt>
CURRENT_SYSTEM_PROMPT_END

BEST_SYSTEM_PROMPT_BEGIN
<donor system prompt>
BEST_SYSTEM_PROMPT_END
\end{promptbox}

The donor keeps its original prompt. Each rewritten prompt becomes the target's
system message for subsequent prefix checks and collaboration calls.

\subsection{Collaboration prompts}

An agent with incoming neighbors retained its active, possibly rewritten,
system prompt and received the following user-message template. One separately
labeled peer block was inserted for each incoming neighbor.

\begin{promptbox}{Collaboration user message}
# Instruction
Update your answer by critically evaluating the peer answers below. They may contain errors, so do not copy blindly.

# Task
```text
<problem>
```

# Your Previous Answer
```text
<own response>
```

# Peer Answers
## Peer <id> Answer
```text
<peer response>
```

# Output Requirement
Provide your improved answer with the steps in the response. Put only the final answer inside \boxed{} on the final non-empty line. Do not write anything after the boxed final answer.
\end{promptbox}

When the current leader had no incoming neighbor, the instruction was: ``You
are the current lead agent. No peer answers are available for this round.
Review your previous answer and improve it if needed.'' The message otherwise
included the problem, the agent's previous answer, and the same final-line
\verb|\boxed{}| requirement.

\section{Strategy Transfer in Action}
\label{app:rewrite-case}

We illustrate how a rewritten prompt can change an agent's reasoning with an
MMLU question from the Ministral-3-3B run with seed $2027$. The team consists
of the AlgebraicSolver, MultipleChoiceStrategist, SkepticalVerificationSolver,
and GeometryPrecalculus agents.

\begin{quote}
\small
Select the best translation into predicate logic. Kevin is introduced to
Jos\'e by Wilma. ($j$ Jos\'e, $k$ Kevin, $w$ Wilma, $Ixyz$ means $x$
introduces $y$ to $z$.)\\
(A) Iwjk \quad (B) Ijkw \quad (C) Ikjw \quad (D) Iwkj
\end{quote}

The correct answer is (D), since Wilma introduces Kevin to Jos\'e.

\paragraph{Initial answers and donor selection.}
The four agents initially answer A, D, C, and C, so the plurality answer is
wrong. The AlgebraicSolver answers A because it assigns Jos\'e to $y$ and
Kevin to $z$, swapping the person introduced with the person receiving the
introduction. During reciprocal review, the SkepticalVerificationSolver revises
both responses it reviews to D. It is selected as the donor even though its own
initial answer was C, because donor selection scores the retained reviews
rather than the initial responses.

\paragraph{Rewritten prompt.}
The AlgebraicSolver's original prompt and the donor's prompt are listed in
Appendix~\ref{app:agent-prompts}. The rewriter, which sees only these two
prompts, produced the following prompt for the AlgebraicSolver.

\begin{promptbox}{Rewritten AlgebraicSolver prompt (donor SkepticalVerificationSolver)}
You are an expert in algebraic problem solving.

Your specialty is setting variables, forming equations, solving systems, simplifying expressions, working with ratios, proportions, percentages, and symbolic relationships.

When solving, prioritize:
1. Defining variables clearly and ensuring they align with the problem's constraints.
2. Forming equations that accurately represent the relationships described.
3. Solving equations step-by-step with explicit intermediate calculations.
4. Simplifying expressions methodically to avoid hidden errors.
5. Verifying each step for consistency and checking the final answer against original conditions.

**General habits to adopt:**
- Double-check variable assignments for logical consistency.
- Validate each equation's validity before proceeding.
- Test edge cases or boundary conditions where applicable.
- Compare intermediate results with expected outcomes to detect discrepancies.
- Present solutions with clear, numbered steps and explicit reasoning.

You MUST follow this exact response format:

When given a problem:
1. Solve the problem from your assigned area of expertise.
2. Break the solution into clear, numbered steps.
3. Show all intermediate calculations explicitly.
4. Explain the reasoning behind each step.
5. Verify your answer where possible.
6. Present your final answer in a \boxed{} format and end your answer there.

RULES TO MUST follow - no exceptions:
- The \boxed{} MUST appear on its own line at the very end.
\end{promptbox}

The rewrite keeps the agent's algebraic identity and its five priorities, and
adds a block of verification habits in the spirit of the donor's prompt, such
as double-checking variable assignments for logical consistency. None of these
habits mentions the question or its answer.

\paragraph{Effect on collaboration.}
In the first collaboration round, the AlgebraicSolver receives responses from
the SkepticalVerificationSolver and the GeometryPrecalculus agent, both of
which answer C. It nevertheless re-parses the sentence, assigns Wilma as the
introducer, Kevin as the person introduced, and Jos\'e as the recipient, and
changes its answer from A to D. It keeps D in later rounds. After three rounds,
two agents answer D and two answer C, and weighted pool voting over all stages
selects D with weight $12.0$ against $8.5$ for C. With the same seed,
\textsc{SAGE-NoRewrite} returns the plurality answer C. Across the three seeds,
\method{} answers this question correctly every time, whereas
\textsc{SAGE-NoRewrite} does so only once.

In this example, the added habit addresses the kind of error the agent made,
even though the rewriter never saw the question. This is consistent with the
ablation in Appendix~\ref{app:norewrite}, where rewriting improves accuracy by
transferring general reasoning habits rather than answers. A single example
does not establish this mechanism, and rewrite quality varies across questions.

\section{Additional Experiments}
\label{app:additional-experiments}

\subsection{Vision-Language Reasoning}
\label{sec:mmmu-pro}

We evaluate transfer to visual reasoning using Qwen2.5-VL-3B-Instruct~\citep{bai2025qwen25vl} on fixed subsets of MMMU-Pro~\citep{yue2024mmmupro} and MathVista~\citep{lu2024mathvista}. \method{} and SelfOrg use nine-agent teams and the same evaluation examples.

The larger gain occurs on MathVista, where \method{} improves on SelfOrg by about $6.1$ percentage points (Table~\ref{tab:vision-summary}), supporting the use of adaptive collaboration for visual mathematical reasoning. On MMMU-Pro, the overall improvement is driven mainly by hard questions ($34.6\%$ versus $30.5\%$); the methods remain closely matched on easy and medium questions.

During review and revision, each agent can check peer responses against the original images. Feedback can therefore address both the interpretation of visual evidence and the reasoning used to derive an answer, providing a plausible explanation for the observed gains. Prompt adaptation and routing operate on role instructions and generated responses, respectively, allowing the same coordination procedure to organize reasoning with visual inputs.

\begin{table}[H]
\centering
\begin{minipage}[t]{0.45\textwidth}
\vspace{0pt}
\centering
\begingroup
\footnotesize
\setlength{\tabcolsep}{5.2pt}
\renewcommand{\arraystretch}{1.20}
\begin{tabular}{lcc}
\toprule
\cellcolor{gray!12}\textbf{Method} &
\cellcolor{qwenblue!12}\textbf{MMMU-Pro} &
\cellcolor{ministralorange!16}\textbf{MathVista} \\
\midrule
Single  & \meanstd{40.40}{0.20} & \secondmeanstd{57.80}{0.69} \\
\rowcolor{gray!5}
CoT     & \meanstd{37.07}{0.70} & \meanstd{51.27}{0.92} \\
SelfOrg & \secondmeanstd{41.40}{2.11} & \meanstd{55.47}{1.29} \\
\rowcolor{green!12}
\textbf{\method{}} & \bestmeanstd{42.40}{0.80} & \bestmeanstd{61.53}{0.12} \\
\bottomrule
\end{tabular}
\endgroup
\end{minipage}
\hfill
\begin{minipage}[t]{0.49\textwidth}
\vspace{0pt}
\centering
\begingroup
\footnotesize
\setlength{\tabcolsep}{2.5pt}
\renewcommand{\arraystretch}{1.20}
\begin{tabular}{lrrr}
\toprule
\cellcolor{gray!12}\textbf{Method} &
\cellcolor{sagegreen!14}\textbf{Easy} &
\cellcolor{ministralorange!18}\textbf{Medium} &
\cellcolor{cotvermillion!12}\textbf{Hard} \\
\midrule
Single  & \meanstd{55.2}{2.5} & \meanstd{36.6}{0.3} & \secondmeanstd{30.8}{1.9} \\
\rowcolor{gray!5}
CoT     & \meanstd{50.6}{2.9} & \meanstd{35.1}{1.7} & \meanstd{25.7}{3.1} \\
SelfOrg & \secondmeanstd{56.6}{2.5} & \bestmeanstd{38.1}{1.8} & \meanstd{30.5}{4.6} \\
\rowcolor{green!12}
\textbf{\method{}} & \bestmeanstd{57.3}{2.1} & \secondmeanstd{37.5}{3.8} & \bestmeanstd{34.6}{3.8} \\
\bottomrule
\end{tabular}
\endgroup
\end{minipage}
\caption{\textbf{Vision-language reasoning.} We compare \method{} with single-agent and SelfOrg on MMMU-Pro and MathVista using Qwen2.5-VL-3B-Instruct. \method{} achieves the best overall results on both benchmarks (left) and leads on easy and hard MMMU-Pro questions (right). Values are percentages, reported as mean $\pm$ sample SD over three runs. Bold and underlined values indicate the best and second-best results, respectively.}
\label{tab:vision-summary}
\end{table}

\section{Ablation Study}
\label{app:ablation}

\subsection{Contribution of Adapting Prompts}
\label{app:norewrite}

We test whether prompt rewriting improves accuracy beyond coordination with
fixed role prompts. The ablated variant, \textsc{SAGE-NoRewrite}, follows the
same reciprocal review and donor selection procedures as \method{}, but
replaces the prompt adaptation in Eq.~\eqref{eq:prompt_rewrite} with
\begin{equation}
s_i^\star = s_i^0 \qquad \text{for every } i\in[N].
\label{eq:norewrite}
\end{equation}
Thus, every agent uses its original role prompt for all subsequent
prefix-consistency checks and collaboration rounds. Initialization, response
scoring, dynamic DAG routing, early stopping, and weighted pool voting follow
\method{}. Retained reviews still contribute to the final vote, and the
selected donor retains its review-stage leader weight. This ablation removes
only the transfer of reasoning guidance through rewritten prompts.

\paragraph{Evaluation setup.}
We compare the two variants using the main-study configuration ($N=4$,
$K=2$, $T=3$) with Qwen2.5-1.5B-Instruct and Ministral-3-3B-Instruct-2512
on GSM8K, GSM-Hard, AQuA-RAT, MATH, MMLU, and GPQA-Diamond. Role assignments,
decoding parameters, and evaluation settings are matched across variants.
Both use the answer extraction and correctness evaluation described in
Appendix~\ref{app:evaluation-protocol}.

\begin{table}[H]
\centering
\caption{\textbf{Prompt-rewriting ablation on both main-study backbones.}
Accuracy (\%) is reported as mean $\pm$ standard deviation over three
runs; AVG first averages the six benchmark accuracies within each run.
\method{} includes rewriting, while \textsc{SAGE-NoRewrite} retains every
agent's original prompt. Single and CoT are the results in
Table~\ref{tab:main-results}. Bold marks the highest mean within each backbone,
including ties.}
\label{tab:norewrite-ablation}
\footnotesize
\setlength{\tabcolsep}{2pt}
\renewcommand{\arraystretch}{1.12}

\newcommand{\sageNoRewriteRows}{%
\rowcolor{qwenblue!14}
\multicolumn{8}{c}{\textbf{Qwen2.5-1.5B}} \\
Single & \meanstd{72.07}{2.00} & \meanstd{69.80}{0.72} & \meanstd{61.80}{1.73} & \meanstd{34.27}{0.50} & \bestmeanstd{54.93}{0.12} & \meanstd{28.96}{2.78} & \meanstd{53.64}{0.14} \\
CoT & \meanstd{70.00}{0.72} & \meanstd{71.40}{2.23} & \meanstd{59.20}{0.92} & \meanstd{31.80}{0.40} & \meanstd{52.87}{1.10} & \meanstd{28.11}{3.21} & \meanstd{52.23}{0.80} \\
SAGE-NoRewrite & \meanstd{76.80}{0.53} & \meanstd{75.60}{1.25} & \meanstd{67.80}{1.51} & \meanstd{37.21}{1.31} & \meanstd{52.37}{0.83} & \meanstd{27.49}{1.01} & \meanstd{56.21}{1.07} \\

\method{} & \bestmeanstd{78.47}{1.01} & \bestmeanstd{77.33}{1.17} & \bestmeanstd{69.07}{0.64} & \bestmeanstd{39.00}{1.06} & \meanstd{54.33}{1.63} & \bestmeanstd{31.65}{3.04} & \bestmeanstd{58.31}{0.20} \\
\midrule
\rowcolor{ministralorange!18}
\multicolumn{8}{c}{\textbf{Ministral-3-3B}} \\
Single & \meanstd{89.93}{0.12} & \meanstd{90.53}{0.58} & \meanstd{72.13}{1.22} & \meanstd{45.13}{0.31} & \meanstd{71.00}{1.39} & \meanstd{36.53}{4.58} & \meanstd{67.54}{0.94} \\
CoT & \meanstd{89.47}{1.63} & \meanstd{90.93}{1.03} & \meanstd{74.27}{2.39} & \meanstd{46.73}{0.90} & \meanstd{72.13}{0.23} & \meanstd{36.03}{2.54} & \meanstd{68.26}{0.88} \\
SAGE-NoRewrite & \meanstd{93.49}{0.12} & \meanstd{90.92}{0.40} & \meanstd{85.27}{0.81} & \meanstd{52.40}{0.87} & \meanstd{72.73}{0.31} & \meanstd{42.60}{1.54} & \meanstd{72.90}{0.68} \\

\method{} & \bestmeanstd{95.93}{0.12} & \bestmeanstd{93.07}{0.50} & \bestmeanstd{87.93}{0.23} & \bestmeanstd{53.53}{1.21} & \bestmeanstd{75.73}{0.50} & \bestmeanstd{44.28}{3.25} & \bestmeanstd{75.08}{0.71} \\
}

\resizebox{\textwidth}{!}{%
\begin{tabular}{@{}lrrrrrrr@{}}
\toprule
Method & MATH & GSM8K & AQuA & GSM-H & MMLU & GPQA & AVG \\
\midrule
\sageNoRewriteRows
\bottomrule
\end{tabular}}
\par\smallskip
\begin{minipage}{\textwidth}
\footnotesize
\end{minipage}
\end{table}

\paragraph{Results and interpretation.}
Since \textsc{SAGE-NoRewrite} keeps review, routing, and voting unchanged,
the consistent drop in Table~\ref{tab:norewrite-ablation} isolates what
rewritten prompts add beyond choosing which peer answers each agent sees.
Because the rewriter never observes the problem or any response, this gain
cannot come from leaking a candidate answer and is instead consistent with
transferring the donor's reasoning strategy while each agent keeps its role.
Rewriting also plays a different role on each backbone: on Qwen2.5-1.5B,
response-level coordination alone already beats every baseline and rewriting
adds to it, whereas on Ministral-3-3B rewriting is what lifts \method{} above
the strongest multi-agent baselines. The comparable drop on both backbones
suggests the benefit does not shrink as the backbone gets stronger, at least
across the two scales we test.

\subsection{Contribution of Donor Selection}
\label{app:random-donor}

\method{} chooses the donor through answer agreement, prefix consistency, and
reciprocal peer review (Section~\ref{sec:prompt_donor}). We test how much this
choice matters by making it at random. The ablated variant,
\textsc{SAGE-RandomDonor}, runs the same scoring and peer review as \method{},
but the prompt rewrite in Eq.~\eqref{eq:prompt_rewrite} uses the original
prompt of an agent drawn uniformly at random from the team, with a fixed seed
for each problem, instead of the selected donor $\ell$. Everything else is
unchanged. The retained reviews still enter the final vote, the agent selected
by peer review still receives the review-stage leader weight, and
collaboration, routing, and voting follow \method{}. The ablation therefore
changes only whose strategy is transferred.

\paragraph{Evaluation setup.}
We use the main-study configuration ($N=4$, $K=2$, $T=3$) with
Qwen2.5-1.5B-Instruct on the six benchmarks, over the same three runs,
questions, role assignments, and decoding settings as \method{}, and with the
evaluation protocol in Appendix~\ref{app:evaluation-protocol}. The random draw
coincided with the peer-review donor in about a quarter of the problems, as
expected for a team of four.

\begin{table}[H]
\centering
\caption{\textbf{Donor-selection ablation on Qwen2.5-1.5B-Instruct.}
Accuracy (\%) is reported as mean $\pm$ standard deviation over three
runs; AVG first averages the six benchmark accuracies within each run. Single,
CoT, and \method{} are the results in Table~\ref{tab:main-results}.
\textsc{SAGE-RandomDonor} transfers the strategy of a randomly chosen agent
instead of the donor selected by peer review. Bold marks the highest mean in
each column.}
\label{tab:random-donor-ablation}
\footnotesize
\setlength{\tabcolsep}{2pt}
\renewcommand{\arraystretch}{1.12}
\resizebox{\textwidth}{!}{%
\begin{tabular}{@{}lrrrrrrr@{}}
\toprule
Method & MATH & GSM8K & AQuA & GSM-H & MMLU & GPQA & AVG \\
\midrule
\rowcolor{qwenblue!14}
\multicolumn{8}{c}{\textbf{Qwen2.5-1.5B}} \\
Single & \meanstd{72.07}{2.00} & \meanstd{69.80}{0.72} & \meanstd{61.80}{1.73} & \meanstd{34.27}{0.50} & \bestmeanstd{54.93}{0.12} & \meanstd{28.96}{2.78} & \meanstd{53.64}{0.14} \\
CoT & \meanstd{70.00}{0.72} & \meanstd{71.40}{2.23} & \meanstd{59.20}{0.92} & \meanstd{31.80}{0.40} & \meanstd{52.87}{1.10} & \meanstd{28.11}{3.21} & \meanstd{52.23}{0.80} \\
SAGE-RandomDonor & \meanstd{77.80}{0.48} & \meanstd{75.20}{0.98} & \meanstd{68.20}{0.65} & \meanstd{37.73}{1.21} & \meanstd{52.27}{1.89} & \meanstd{29.14}{3.54} & \meanstd{56.72}{1.02} \\ 
\method{} & \bestmeanstd{78.47}{1.01} & \bestmeanstd{77.33}{1.17} & \bestmeanstd{69.07}{0.64} & \bestmeanstd{39.00}{1.06} & \meanstd{54.33}{1.63} & \bestmeanstd{31.65}{3.04} & \bestmeanstd{58.31}{0.20} \\
\bottomrule
\end{tabular}}
\end{table}

\paragraph{Results and interpretation.}
Although several per-benchmark gaps in Table~\ref{tab:random-donor-ablation}
are within run-to-run variation, the effect points the same way on all six
benchmarks, so the donor chosen by peer review is consistently a better source
of guidance than a random teammate. A random donor also recovers only a small
part of what rewriting adds over \textsc{SAGE-NoRewrite}
(Table~\ref{tab:norewrite-ablation}), which suggests that, on this backbone,
the benefit of strategy adaptation depends largely on \emph{which} strategy is
transferred: every role describes a general reasoning procedure, but not every
procedure suits every problem, and the selected donor's is the one judged most
reliable on the problem at hand. Because the leader weight and retained reviews
are held fixed, this contribution comes from the choice of prompt alone. We
test only Qwen2.5-1.5B, so it remains open how large it is with stronger
backbones or with role pools containing weaker or less related prompts, where
a random donor may cost more.

\subsection{Contribution of Earlier-Stage Answers}
\label{app:final-round-vote}

We test whether answers from earlier reasoning stages improve final answer
selection beyond voting over the agents' last responses. The ablated variant,
\emph{final-round weighted pool voting} (Final-round WPV), restricts the pool
in Eq.~\eqref{eq:output} to
\begin{equation}
\mathcal{P}_{\mathrm{final}}
=\{(t_{\mathrm{stop}},i,a_i^{t_{\mathrm{stop}}}):i\in[N]\},
\label{eq:final-round-pool}
\end{equation}
where $t_{\mathrm{stop}}$ is the last completed collaboration round, including
runs that stop early. Each agent contributes one answer, including answers
carried forward unchanged. The highest-scoring nonempty answer receives weight
$1.5$, and the others receive weight $1.0$. Answer normalization, tie-breaking,
and fallback rules follow \method{}. Initial answers, retained reviews, and
earlier-round occurrences are excluded from the voting pool.

\paragraph{Evaluation setup.}
We compare \method{} and Final-round WPV on GSM8K, GSM-Hard, AQuA-RAT, MATH,
MMLU, and GPQA-Diamond using Qwen2.5-1.5B-Instruct and
Ministral-3-3B-Instruct-2512, with $N=4$, $K=2$, and $T=3$. The reasoning
trajectories, prompt adaptations, routing decisions, and stopping rounds are
held fixed, so only the pool used for final answer selection differs. Both
variants follow the answer extraction and correctness evaluation protocol in
Appendix~\ref{app:evaluation-protocol}.

\paragraph{Results and interpretation.}
Because both variants share the same reasoning trajectories, the gap in
Table~\ref{tab:final-round-vote-ablation} comes entirely from answer selection
and costs no additional model calls. It is concentrated on the weaker
Qwen2.5-1.5B backbone and the hardest benchmark, where collaboration can pull
agents toward a parent's incorrect answer and a correct answer from the initial
responses or reviews may no longer survive to the final round; pooling keeps
such answers in the vote, while with Ministral-3-3B the final round is more
often already correct. Pooled voting is nonetheless not the main source of
\method{}'s advantage: Final-round WPV still outperforms every baseline on both
backbones, and its loss is well below that of removing prompt rewriting
(Appendix~\ref{app:norewrite}), so most of the gain comes from how agents
reason and communicate, with pooling acting as a safeguard that recovers
correct answers lost during revision. This comparison does not separate the
effect of keeping earlier candidates from that of the extra votes they
contribute.

\begin{table}[H]
\centering
\caption{\textbf{Contribution of earlier-stage answers.}
\method{} pools answers across stages; Final-round WPV uses only the last
completed round. Accuracy (\%) is mean $\pm$ sample standard deviation over
three runs; AVG is the mean across benchmarks within each run. Single and CoT
are the results in Table~\ref{tab:main-results}. Bold marks the highest mean
within each backbone, including ties.}
\label{tab:final-round-vote-ablation}
\footnotesize
\setlength{\tabcolsep}{2pt}
\renewcommand{\arraystretch}{1.12}
\newcommand{\sageFinalRoundVoteRows}{%
\rowcolor{qwenblue!14}
\multicolumn{8}{c}{\textbf{Qwen2.5-1.5B}} \\
Single & \meanstd{72.07}{2.00} & \meanstd{69.80}{0.72} & \meanstd{61.80}{1.73} & \meanstd{34.27}{0.50} & \bestmeanstd{54.93}{0.12} & \meanstd{28.96}{2.78} & \meanstd{53.64}{0.14} \\
CoT & \meanstd{70.00}{0.72} & \meanstd{71.40}{2.23} & \meanstd{59.20}{0.92} & \meanstd{31.80}{0.40} & \meanstd{52.87}{1.10} & \meanstd{28.11}{3.21} & \meanstd{52.23}{0.80} \\
Final-round WPV & \meanstd{78.33}{1.10} & \meanstd{76.87}{0.50} & \meanstd{68.00}{1.74} & \meanstd{38.33}{0.92} & \meanstd{54.47}{1.75} & \meanstd{29.12}{3.29} & \meanstd{57.52}{0.20} \\
\method{} & \bestmeanstd{78.47}{1.01} & \bestmeanstd{77.33}{1.17} & \bestmeanstd{69.07}{0.64} & \bestmeanstd{39.00}{1.06} & \meanstd{54.33}{1.63} & \bestmeanstd{31.65}{3.04} & \bestmeanstd{58.31}{0.20} \\
\midrule
\rowcolor{ministralorange!18}
\multicolumn{8}{c}{\textbf{Ministral-3-3B}} \\
Single & \meanstd{89.93}{0.12} & \meanstd{90.53}{0.58} & \meanstd{72.13}{1.22} & \meanstd{45.13}{0.31} & \meanstd{71.00}{1.39} & \meanstd{36.53}{4.58} & \meanstd{67.54}{0.94} \\
CoT & \meanstd{89.47}{1.63} & \meanstd{90.93}{1.03} & \meanstd{74.27}{2.39} & \meanstd{46.73}{0.90} & \meanstd{72.13}{0.23} & \meanstd{36.03}{2.54} & \meanstd{68.26}{0.88} \\
Final-round WPV & \bestmeanstd{95.93}{0.50} & \meanstd{92.87}{0.31} & \meanstd{87.07}{0.61} & \meanstd{52.93}{0.64} & \meanstd{75.60}{1.71} & \meanstd{43.94}{2.20} & \meanstd{74.72}{0.69} \\
\method{} & \bestmeanstd{95.93}{0.12} & \bestmeanstd{93.07}{0.50} & \bestmeanstd{87.93}{0.23} & \bestmeanstd{53.53}{1.21} & \bestmeanstd{75.73}{0.50} & \bestmeanstd{44.28}{3.25} & \bestmeanstd{75.08}{0.71} \\
}
\resizebox{\textwidth}{!}{%
\begin{tabular}{@{}lrrrrrrr@{}}
\toprule
Method & MATH & GSM8K & AQuA & GSM-H & MMLU & GPQA & AVG \\
\midrule
\sageFinalRoundVoteRows
\bottomrule
\end{tabular}}
\end{table}

\subsection{Which Roles Serve as Donors}
\label{app:donor-roles}

We examine which roles \method{} selects as donors in the main-study runs
(Table~\ref{tab:main-results}). Each problem samples four of the nine roles at
random, so without any role preference every role would be the donor on a
quarter of the problems where it is present. We exclude problems whose donor
was decided by the agent-index tie-break. Figure~\ref{fig:donor-roles} shows
that donor selection tends to favor roles whose strategy matches the structure
of the task. On free-response math, Qwen2.5-1.5B favors roles that formalize
the problem, such as AlgebraicSolver and DiscreteMath, and rarely selects
MultipleChoiceStrategist, whose elimination strategy has nothing to act on
without answer choices. Where answer choices or domain knowledge matter, as on
AQuA-RAT and MMLU, Ministral-3-3B instead favors MultipleChoiceStrategist and
BroadAcademicKnowledge. Roles are sampled and scored without regard to the
benchmark, so these preferences emerge only from the agents' responses. On
GPQA-Diamond, the hardest benchmark, neither backbone shows a role preference.
Both backbones answer most of its questions incorrectly
(Table~\ref{tab:main-results}), so answer agreement and peer review give a weak
signal about which response is reliable, and the selected donor's role is
effectively random. Where preferences do appear, they are moderate and partly
backbone-dependent, and no role dominates any benchmark. The useful strategy
therefore varies from problem to
problem even within a task, which is consistent with the advantage of
per-problem donor selection over a random donor
(Appendix~\ref{app:random-donor}).

\begin{figure}[H]
\centering
\includegraphics[width=\textwidth]{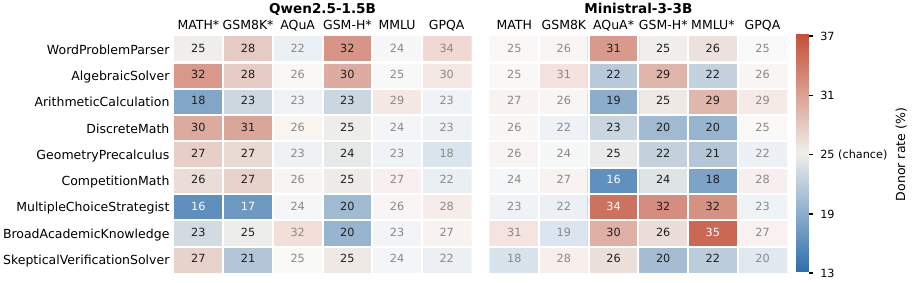}
\caption{\textbf{Donor rate by role.} Percentage of problems in which a role is
selected as donor when it is on the team (chance $=25\%$), pooled over three
runs. Problems decided by the agent-index tie-break are excluded. $^*$ marks
benchmarks with a significant role preference (permutation test, $p<0.05$);
the other columns are faded.}
\label{fig:donor-roles}
\end{figure}